\documentclass{article}

\usepackage{PRIMEarxiv}

\usepackage[utf8]{inputenc} 
\usepackage[T1]{fontenc}    
\usepackage{hyperref}       
\usepackage{url}            
\usepackage{booktabs}       
\usepackage{amsfonts}       
\usepackage{nicefrac}       
\usepackage{microtype}      
\usepackage{lipsum}
\usepackage{fancyhdr}       
\usepackage{graphicx}       
\graphicspath{{media/}}     
\usepackage[numbers]{natbib}
\usepackage{amsmath}
\usepackage{amsthm}
\usepackage{booktabs}
\usepackage{caption}
\usepackage{subcaption}
\usepackage{bold-extra}
\usepackage{nicematrix}
\DeclareMathOperator{\Tr}{Tr}
\newcommand{\bx}{\mathbf{x}}
\newcommand{\bxF}{\mathbf{x}^\mathcal{F}}
\newcommand{\bu}{\mathbf{u}}
\newcommand{\bz}{\mathbf{z}}
\newcommand{\by}{\mathbf{y}}
\newcommand{\calV}{\mathcal{V}}
\newcommand{\calF}{\mathcal{F}}
\usepackage[acronym, nogroupskip, nonumberlist]{glossaries}
\glsdisablehyper
\newacronym{map}{MAP}{maximum a posteriori}
\newacronym{minlp}{MINLP}{mixed-integer nonlinear program}
\newacronym{gmm}{GMM}{Gaussian mixture model}
\newacronym{bgmm}{BGMM}{Bayesian Gaussian mixture model}
\newacronym{baron}{BARON}{Branch-and-Reduce Optimization Navigator}
\newacronym{gbd}{GBD}{generalized Benders' decomposition}
\newacronym{lda}{LDA}{Latent Dirichlet allocation}
\newtheorem{theorem}{Theorem}
\newtheorem{corollary}{Corollary}
\newtheorem{lemma}{Lemma}

\title{Maximum a Posteriori Inference for Factor Graphs via Benders' Decomposition
\thanks{\textit{\underline{Citation}}: 
\textbf{Authors. Title. Pages.... DOI:000000/11111.}}

}

\author{
  Harsh Vardhan Dubey  \\
  Department of Mathematics \& Statistics \\
  University of Massachusetts, Amherst \\
  Amherst\\
  \texttt{hdubey@umass.edu} \\
     \And
  Ji Ah Lee \\
  Department of Mathematics \& Statistics \\
  University of Massachusetts, Amherst \\
  Amherst\\
  \texttt{jlee@umass.edu} \\
   \And
  Patrick Flaherty \\
  Department of Mathematics \& Statistics \\
  University of Massachusetts, Amherst \\
  Amherst\\
  \texttt{pflaherty@umass.edu} \\
}

\begin{document}
\maketitle

\begin{abstract}
Many Bayesian statistical inference problems come down to computing a maximum a-posteriori (MAP) assignment of latent variables.
Yet, standard methods for estimating the MAP assignment do not have a finite time guarantee that the algorithm has converged to a fixed point.
Previous research has found that MAP inference can be represented in dual form as a linear programming problem with a non-polynomial number of constraints.
A Lagrangian relaxation of the dual yields a statistical inference algorithm as a linear programming problem.
However, the decision as to which constraints to remove in the relaxation is often heuristic.
We present a method for maximum a-posteriori inference in general Bayesian factor models that sequentially adds constraints to the fully relaxed dual problem using Benders' decomposition.
Our method enables the incorporation of expressive integer and logical constraints in clustering problems such as must-link, cannot-link, and a minimum number of whole samples allocated to each cluster.
Using this approach, we derive MAP estimation algorithms for the Bayesian Gaussian mixture model and latent Dirichlet allocation.
Empirical results show that our method produces a higher optimal posterior value compared to Gibbs sampling and variational Bayes methods for standard data sets and provides certificate of convergence.
\end{abstract}


\section{Introduction}
In many applications, the maximum a-posteriori (MAP) assignment is the primary object of inference.
For example, decoding error-correcting codes, identifying protein coding regions from DNA sequence data, part-of-speech tagging, and image segmentation~\citep{wainwright2008graphical}.
To address these applications, we have two objectives for the optimal MAP estimation algorithm.
First, it should produce an estimate that is in support of the actual posterior distribution rather than an approximation that must then be rounded to a feasible estimate.
Second, it should accommodate ``hard'' constraints such as logical constraints or must-link constraints that are enforced regardless of the data.
Third, it should provide a finite-time certificate of convergence - that is, it should report that it has converged to a local optimum.
Even without these requirements, however, MAP inference is known to be NP-hard except for graphs of special structure~\citep{shimony1994finding}.
Therefore, we expect to pay a computational price for the convenience of these conditions, but we aim to minimize the computational impacts and retain the ability to solve problems of a size that are of practical importance.

The first condition on the estimation algorithm is that the MAP estimate is a feasible point in the true posterior distribution.
In the case of clustering models, the assignment variables are discrete and the MAP estimation problem is typically a mixed-integer nonlinear programming problem~\citep{murphy2012machine}.
Simulated annealing~\citep{szu1987nonconvex} and genetic algorithms~\citep{hajela1990genetic} converge to the global optimum for non-convex problems, but the convergence guarantees are generally asymptotic rather than finite-time, and assessing convergence is challenging~\citep{brooks1998general}.  
\citep{zymnis2008mixed} addresses the issue of the non-convexity of the MAP problem domain by approximating the domain with a convex hull, solving the nonlinear program over a continuous domain, and then rounding to a feasible MAP estimate.
Randomized rounding can significantly improve the expected performance in some non-convex problems~\citep{goemans1995improved}, however, rounding can lead to a suboptimal solution that is distant from the true MAP value~\citep{land1960automatic}.
Our focus in this work is on the optimal MAP estimate for the particular data set in hand rather than a high probability of good performance across data sets.
     
The second condition we would like to incorporate into the estimation algorithm is the ability to add both hard constraints to restrict the domain and soft constraints to guide the optimal solution based on prior information.
\citep{baumann2020binary} shows that incorporating hard constraints can dramatically improve both accuracy (comparing to known labels) and computational time because the search space is significantly reduced.
Other clustering algorithms can adaptively incorporate user feedback as hard constraints which can also improve clustering accuracy and timing results~\citep{angel2022cluster}.
Most clustering algorithms, however, incorporate must-link and cannot-link constraints as soft constraints so as to simplify the search space~\citep{davidson2005clustering}.
Therefore, incorporating both hard and soft constraints can improve clustering performance and computational efficiency. 

The third condition is that our algorithm should have a finite-time certificate of convergence.
Gradient-based clustering algorithms have good computational timing performance~\citep{murphy2012machine} and can have good average population performance~\citep{balkrishnan2017EM}. 
However, such clustering algorithms are not able to distinguish between convergence towards a point near the global optimum, or to a point near a bad local optimum.
Moreover, they also require suitable initialization values which are not always available for a high dimensional non-convex problem.
Therefore, developing an algorithm that can guarantee convergence even to a local optimum in finite time becomes important.

The rest of the paper is structured as follows. 
In section~\ref{sec:gbd_sec} we derive Generalized Benders’ Decomposition for factor graphs, present a theorem outlining when our proposed algorithm yields an $\epsilon$-global optimal value, and detail a step-wise implementation of our algorithm. 
In section~\ref{sec:bgmm_sec}  we derive and implement our algorithm for the Bayesian Gaussian Mixture Model (BGMM) and present experimental results on three standard data sets by comparing our proposed algorithm to other standard Bayesian inference methods. 
Finally, in section~\ref{sec:lda_sec} we derive and implement our algorithm on the Latent Dirichlet Allocation(LDA) model and present experimental results on the 20 newsgroups data set by comparing our proposed algorithm to other standard Bayesian inference methods. 
The remainder of this introduction is concerned with the general MAP inference problem setup for factor graphs and related work on inference algorithms for factor graphs.

\subsection{Problem Setup}
The maximum a-posteriori (MAP)  estimate is a solution to
\begin{equation}
    \label{eqn:map1}
	\max_\bx \log p(\bx | \by; \phi),
\end{equation}
where $\bx$ is the vector of latent variables, $\by$ is the observed data, and $\phi$ is the vector of parameters of a Bayesian model, $p$.
Latent variables, $\bz$, that are in the model, but not subjects of inference are marginalized as $p(\bx | \by; \phi) = \int_{\bz} p(\bx,\bz | \by; \phi) d\bz$.

The log posterior distribution can always be decomposed into the sum of local interactions~\citep{sontag2010introduction},
\begin{equation}
    \log p(\bx | \by; \phi) = \sum_{v \in \mathcal{V}} \theta_v(x_v) + \sum_{f \in \mathcal{F}} \theta_f(\bx_f),
\end{equation}
where $\calV$ is the set of latent variables, $\calF$ is the set of factors, and $\bx_f$ is the subset of the latent variables that are involved in factor $f$.
The factorization separates the joint log density function into singleton functions, $\theta_v(x_v)$, which only depend on a single element of $\calV$, and factor functions, $\theta_f(\bx_f)$, which depend on a subset $f \in \calV$ of latent variables.
The functions $\theta_v$ and $\theta_f$ incorporate the fixed data, $\by$, and parameters, $\phi$.
The resulting MAP inference problem in factorized form is
\begin{equation}
	\label{eqn:map}
	\text{MAP}(\theta) = \max_\bx \sum_{v \in \mathcal{V}} \theta_v(x_v) + \sum_{f \in \mathcal{F}} \theta_f(\bx_f).
\end{equation}

To construct the dual decomposition~\citep{sontag2010introduction}, the model is augmented by duplicating all the latent variables that are involved in factors $\bx_f$ $\forall f \in \calF$; we denote this set of factor variables $\bx^\mathcal{F} = \{\bx_f : f \in \calF \}$, where $\bx_f = \{x_v : x_v \in f\}$.
Constraints are added to tie the newly created factor variables to the original latent variables.
\begin{eqnarray}
	\label{eqn:map-augmented}
	\text{MAP}(\theta) &=& \max_{\bx, \bx^\mathcal{F}} \sum_{v \in \mathcal{V}} \theta_v(x_v) + \sum_{f \in \mathcal{F}} \sum_{v \in f} \theta_f(\bx_v^f) \\
	\text{subject to} && x_v = x_v^f,\quad \forall v \in f, f \in \mathcal{F}
\end{eqnarray}
The constraints, $x_v = x_v^f$, are ``complicating'' because were it not for those constraints, the problem would be a convex optimization problem if $\theta_v$ and $\theta_f$ were concave and the domain of $\bx$ was convex.
Indeed, many algorithms for MAP inference can be formulated as relaxations of the constraints, the domain, and the objective function~\citep{wainwright2008graphical}.

Having represented the MAP inference problem as a factor graph, the Lagrange dual can then be formulated with the Lagrange function: 
\begin{equation}
	L(\delta, \bx, \bx^\mathcal{F}) = \sum_{v \in \mathcal{V}} \theta_v(x_v) + \sum_{f \in \mathcal{F}} \sum_{v \in f} \theta_f(\bx_v^f) + \sum_{f \in \mathcal{F}} \sum_{v \in f} \delta_{f_v} (x_v - x_v^f),
\end{equation}
where  $\delta_{f_v}$ is Lagrange variable associated with the constraint $g_{f_v}(x_v, x_v^f) := x_v - x_v^f = 0$.

The MAP estimation problem can be solved by relaxing (removing) all of the constraints and solving the dual problem.
The resulting optimization problem completely separates and the size is proportional to the largest factor.
However, the fully relaxed solution may be non-feasible for the original problem.

The problem is then how should we select constraints to add back to the fully relaxed Lagrange dual problem.
\citep{Sontag2008message} presents an approach that starts with a fully relaxed dual problem and adds constraints back in according to a clustering rule.
However, that approach is heuristic and only developed for discrete latent variables.
Here, we present a method based on Benders' decomposition that optimally selects violated constraints and iteratively tightens the relaxation while maintaining rigorous bounds on the value of the optimum.

\subsection{Contributions}
This work has three main contributions.
First, we derive a dual decomposition MAP inference algorithm that preserves hard domain constraints and produces feasible points in the posterior without relaxations of the domain to a convex set.
Second, we show theoretically and empirically that the algorithm yields a certificate of convergence to the $\epsilon$-local optimal value in finite time.
Third, we evaluate the algorithm on two common Bayesian models: the Bayesian Gaussian mixture model, and latent Dirichlet allocation.
A comparison with Gibbs sampling and variational inference methods shows that our proposed method produces a better optimal value on common test data sets than these other methods.

\subsection{Related Work}

We now provide a brief overview of three widely used methods for maximum a-posteriori (MAP) inference for Bayesian factor models.

Gibbs sampling is an algorithm to generate a sequence of samples from an arbitrary multivariate probability distribution without requiring the normalization constant~\citep{geman1984gibbs, chib1995understanding}. 
It is particularly well-adapted to sampling the posterior distribution of a Bayesian network since Bayesian networks are typically specified as a collection of conditional distributions and the sampler samples each variable, in turn, conditional on the current value of other variables.
The usefulness and practical applicability of Gibbs sampling depends on the ease with which samples can be drawn from the conditional distributions.
If the graph of a Bayesian network is constructed using distributions from the exponential family, and if the parent-child relationships preserve conjugacy, then sampling from the conditional distributions that arise in Gibbs sampling can be done with a known distribution function~\citep{bishop2006pattern}.

The generality of Gibbs sampling comes with some costs.
It is guaranteed to converge in distribution to the posterior, but assessing when that convergence has occurred is difficult~\citep{brooks1998general, gelman1992inference}.
It naturally provides an estimate of the moments of a posterior distribution, but estimating the mode of a high dimensional distribution is challenging because simple averaging does not provide a feasible point if the domain is non-convex~\citep{andrieu2003introduction}.
In clustering models, permutations of the cluster assignment labels have equivalent posterior probabilities.
So, sampling algorithms can ``mode-hop'' which produces samples from different domains, but without practical differences in meaning of the different sample values~\citep{sminchisescu2003mode}.
Finally, sampling methods often become highly inefficient or completely unable to sample from high-dimensional discrete random variables.
The use of a momentum term which requires a derivative improves the efficiency, but limits the applicability of the sampler for clustering applications unless special transformations are employed~\citep{neal2011mcmc}.

Variational Bayesian methods approximate the complicated posterior distribution, $p(\bx|\by; \phi)$, with a simpler surrogate distribution $q(\bx)$ such that the KL divergence between the two, $\text{KL}(q \| p)$, is minimized~\citep{bishop2006pattern}.

When $q$ is taken as a parametric family, the minimization is over the finite set of parameters.
Given the best-fit approximation, the mode of the density (the MAP value) can be obtained using search methods or analytical values if the distribution is simple enough.
It is scalable to large data sets and applicable to complex models where it would be too computationally expensive~\citep{tran2021bayes}.
The minimization of $\text{KL}(q \| p)$ is accomplished by iterative updates to a fixed point; the update equations typically make use of gradient information~\citep{blei2017variational}.
There are some notable results on the consistency of variational inference~\citep{wang2019frequentist, pati2018statistical}.

There are two drawbacks of variational inference methods for MAP estimation.
First, the choice of the surrogate distribution is a tradeoff between computability and fidelity to the true posterior~\citep{wainwright2008graphical}. 
A simple surrogate is easy to optimize, but can yield a MAP estimate that is distant from the true value.
It is difficult to assess what is and is not a good surrogate function.
Second, convergence to the fixed point is typically assessed by the change in the value of the surrogate parameters~\citep{blei2017variational}.
The approach does not provide an upper bound on the optimal value to assess whether better values of the parameters of $q$ may be available.
The algorithm is typically rerun from multiple initial points to guard against convergence to local optima, but selecting initial points is challenging in high dimensional parameter space.

Since the MAP problem is a mixed-integer nonlinear program in general, methods for solving that general class of problems can be used for MAP inference. 
The cutting plane method~\cite{bertsimas2005optimization} was used in \citep{Sontag2008message} to solve the MAP inference problem.
However, the problem considered in that work only involved discrete variables.
When there are continuous variables in the problem, a separate optimization problem is solved for each setting of the discrete variables and no information is passed from the continuous solution to the discrete solver~\citep{gilmore1961linear}.
As a result, for MINLP problems, cutting plane methods can be inefficient. 

Branch and bound methods use a divide and conquer approach to search the feasible set.
In the case of the maximization problem, it uses an upper bound on the optimal cost to exclude parts of the domain and focus the search on areas where improvement is possible~\citep{bertsimas2005optimization, land1960automatic}.
Branch and bound algorithms work by breaking down an optimization problem into a series of smaller sub-problems and using a bounding function to eliminate sub-problems that cannot contain the optimal solution~\citep{clausen1999branch}.
Branch and bound methods have steadily improved in terms of computational efficiency and can now handle very large problem sizes~\citep{bertsimas2017logistic}.
\gls{baron} which is based on the branch and bound philosophy is the state-of-the-art solver for general \glspl{minlp}~\citep{kronqvist2019review}.
It exploits the general problem structure to Branch and bound \glspl{minlp} with domain reduction on both discrete and continuous decision variables~\citep{sahinidis2017baron}.
While branch and bound methods can be used for mixed integer nonlinear programming problems, they seem to be best suited for nonlinear programming problems with no discrete variables because the branching process can require very deep searches.
In the case of clustering, moving a single data point from one cluster to another can have a very small impact on the optimal value and thus the branching process is not able to exclude large parts of the domain close to the root of the branch and bound tree.

Outer approximation methods require the user to explicitly separate the problem into integer and continuous variables~\citep{duran1986outer, floudas2013deterministic}.
Once this is done, the integer variables are held fixed while the solution to a simpler continuous problem is solved.
The outer approximation serves as an upper bound on the maximum of the true objective.
By sequentially tightening the bound, the lower and upper bounds converge to a fixed point.

Generalized Benders' decomposition provides a solution to certain mixed-integer nonlinear programming problems by separating the original problem variables into ``complicating'' and ``not-complicating variables''~\citep{geoffrion1972generalized}.
The original optimization problem is decomposed into a so-called ``master'' problem and ``sub-problem''.
The Lagrange multiplier from the subproblem is used to form constraints on the fully relaxed master problem and the approximation of the original problem is tightened as constraints are added.
In this way, the method resembles the outer approximation method, but there are meaningful differences~\citep{duran1984mixed}.
Generalized Benders' decomposition has been extended to a more general class of mixed-integer nonlinear programs and has been used to solve large-scale problems in chemical engineering and other domains~\citep{floudas2013deterministic}.

\section{Generalized Benders' Decomposition for Factor Graphs}
\label{sec:gbd_sec}

In this section, the Generalized Benders' decomposition approach is extended to problems in statistical inference - in particular maximum a-posterior inference in factor graph models.
As stated, Generalized Benders' decomposition~\citep{geoffrion1972generalized} is used for solving problems of the form
\begin{equation}
\begin{aligned}
	\label{eqn:gbd}
	\max_{x,y} & f(x,y) \\
	\text{subject to } & G(x,y) \geq 0, \quad x \in \mathcal{X}, y \in \mathcal{Y},
\end{aligned}
\end{equation}
where the following conditions on the objective and constraints hold:
\begin{enumerate}
	\item for a fixed $y$, $f(x,y)$ separates into independent optimization problems each involving a different subvector of $x$,
	\item for a fixed $y$, $f(x,y)$ is of a special structure that can be solved efficiently, and 
	\item fixing $y$ renders the optimization problem convex in $x$.
\end{enumerate}
The constraint function $G(x,y)$ captures all of the constraints that involve \textit{both} $x$ and $y$, while the constraints $x \in \mathcal{X}$ and $y \in \mathcal{Y}$ capture the constraints that involve either variable but not both simultaneously.

\begin{theorem}
If the posterior is convex in the latent variables, maximum a-posteriori inference for Bayesian models in factor graph form satisfies the conditions of generalized Benders' decomposition, and thus global optimality can be achieved.
\end{theorem}

\begin{proof}
The factor graph representation of maximum a-posteriori inference for a Bayesian model is
\begin{eqnarray*}
	\text{MAP}(\theta) &=& \max_{\bx, \bx^\mathcal{F}} \sum_{v \in \mathcal{V}} \theta_v(x_v) + \sum_{f \in \mathcal{F}} \sum_{v \in f} \theta_f(\bx_v^f) \\
	\text{subject to} && x_v = x_v^f,\quad \forall v \in f, f \in \mathcal{F}
\end{eqnarray*}
and the general form for generalized Benders' decomposition is 
\begin{eqnarray*}
	\max_{x,y} & f(x,y) \\
	\text{subject to } & G(x,y) \geq 0, \quad x \in \mathcal{X}, y \in \mathcal{Y}.
\end{eqnarray*}
Let $x^\calF$ in the MAP inference problem be $x$ in the generalized Benders' decomposition problem, and let $\bx$ in the MAP inference problem be $y$ in the generalized Benders' decomposition problem.

Condition 1 is satisfied because if $\bx$ in the MAP problem is held fixed, the maximization is over a sum of factors, each of which is a subvector of $\bx^\calF$.
If $\bx$ is fixed, the optimization problem is trivial because the constraints $x_v = x_v^f,\quad \forall v \in f, f \in \mathcal{F}$ set the value of $\bx^\calF$; constraint 2 is satisfied.
Finally, the condition that the posterior be convex in the latent variables renders it convex in $\bx^\calF$ for fixed $\bx$.
\end{proof}

\begin{corollary}
\label{benders-corollary}
If the posterior is not convex in the latent variables, maximum a-posteriori inference for Bayesian models in factor graph form does not satisfy the conditions of generalized Benders' decomposition, and thus local optimality may be achieved.
\end{corollary}

Generalized Benders' decomposition separates the \emph{complicating} variables in the optimization problem with the idea that holding these variables fixed renders the optimization problem much simpler.
In the Bayesian factor graph model formulation, the original latent variables $\bx$ are construed as complicating variables.
The factor variables, $\bx^\calF$ 
- the clones that were created in the formation of the factor graph - are the non-complicating variables.
The approach is best demonstrated by examples.
First, we present the general purpose algorithm. Then in the next sections, we derive specific MAP inference algorithms for the Bayesian Gaussian mixture model and the latent Dirichlet allocation model.
Proofs and technical details of the specific derivations are found in the Appendices.

\subsection{Algorithm}

Benders' decomposition applied to the factor graph yields an algorithm for solving the problem~\ref{eqn:map}.
\begin{enumerate}
	\item Pick a feasible point $\bar{\bx}$. Solve the subproblem $(1-\bar{\bx})$,
	\begin{equation}
		\max_{\bxF} \sum_{f \in \mathcal{F}} \theta_f (\bx_f^f), \quad \text{s.t.}\quad \bar{x}_v = x_v^f,\ \forall v \in f, f \in \mathcal{F}.
	\end{equation}
	Obtain the optimal multiplier vector $\bar{\bu}$ and the function $L^*(\bx, \bar{\bu})$, where
	\begin{equation*}
		L^*(\bx; \bar{\bu}) = \max_{\bxF} \left\{ \sum_{v \in \mathcal{V}} \theta_v(x_v) + \sum_{f \in \mathcal{F}} \sum_{v \in f} \theta_f(\bx^f) + \bar{\bu}^T G(\bx, \bxF ) \right\}.
	\end{equation*}
	Since the constraints are all of the form $x_v = x_v^f$, the function $L^*$ has the following decomposition
	\begin{equation}
		\sum_{v \in \mathcal{V}} \theta_v(x_v) + \sum_{f \in \mathcal{F}} \sum_{v \in f} \bar{u}_{x_v^f} x_v + \max_{\bxF} \left\{  \sum_{f \in \mathcal{F}} \theta_f(\bx^f) - \sum_{f \in \mathcal{F}} \sum_{v \in f} \bar{u}_{x_v^f} x_v^f \right\}.
	\end{equation}
	And, since the objective is separable, the optimization problem can be solved efficiently.
	Set $p=1$, $q=1$, $u^1 = \bar{u}$, and $\text{LBD} = v(\bar{x})$.
	
	\item Solve the \textit{relaxed master problem}
	\begin{eqnarray*}\label{eqn:master-problem}
		\max_{\bx, x_0}\quad  & x_0& \\
		\text{subject to}\quad  & x_0 & \leq L^*(\bx, \bu^j) , \quad j = 1, \ldots, p \\
		& 0 &\leq L^*(\bx, \mathbf{\lambda}^j), \quad j = 1, \ldots, q.
	\end{eqnarray*}
	Let $(\hat{\bx}, \hat{x}_0)$ be an optimal solution; $\hat{x}_0$ is an upper bound on the optimal value of \eqref{eqn:gbd}.
	If $\text{LBD} \geq \hat{x}_0 - \epsilon$, then terminate.
	
	\item Solve the revised \textit{subproblem} $(1-\hat{\bx})$.
	One of two outcomes is possible:
	\begin{enumerate}
		\item \textbf{The quantity $v(\hat{\bx})$ is finite.}
		If $v(\hat{\bx}) \geq x_0 - \epsilon$, then terminate.
		Otherwise, determine the optimal Lagrange multiplier $\hat{\bu}$ (or a near-optimal multiplier), and the function $L^*(\bx, \hat{\bu})$.
		Increment $p$ by $1$ and set $\bu^p = \hat{\bu}$.
		If $v(\hat{\bx}) > \text{LBD}$, set update the lower bound, $\text{LBD} = v(\hat{\bx})$.
		Return to step 2.
		
		\item \textbf{Problem $(1- \hat{\bx})$ is infeasible.}
		Determine $\hat{\mathbf{\lambda}}$ and the function $L_*(\bx; \hat{\mathbf{\lambda}})$.
		Increase $q$ by $1$ and set $\mathbf{\lambda}^q = \hat{\mathbf{\lambda}}$.
		Return to step 2.
	\end{enumerate}
\end{enumerate}
This algorithm selects the \textit{most violated constraint} as assessed by the Lagrange multiplier and adds that constraint to the master problem.
In this way, Benders' decomposition addresses the challenge of selecting constraints in the variational approximation of factor graph models.

\section{Bayesian Gaussian Mixture Model}
\label{sec:bgmm_sec}
In this section, we formulate and derive a Generalized Benders decomposition algorithm for the Bayesian Gaussian mixture model.

\subsection{Problem Formulation} \label{sec:bgmm-problem}
A Bayesian Gaussian mixture model has the form:
\begin{eqnarray}
	y_i | z_i, \mu, \Lambda &\sim& \text{Gaussian}(\mu_{z_i}, \Lambda^{-1}_{z_i}) \\
	z_i | \pi &\sim& \text{Categorical}_K(\pi) \\
	\pi | \alpha_0 &\sim& \text{Dirichlet}_K(\alpha_0) \\
	\mu_k | \Lambda_k &\sim&  \text{Gaussian}(\mu_0, (\beta_0 \Lambda_k)^{-1}) \\
	\Lambda_k &\sim& \text{Wishart}(W_0, \nu_0)
\end{eqnarray}

\begin{figure}[h!]
	\centering
	\includegraphics[width=0.5\linewidth]{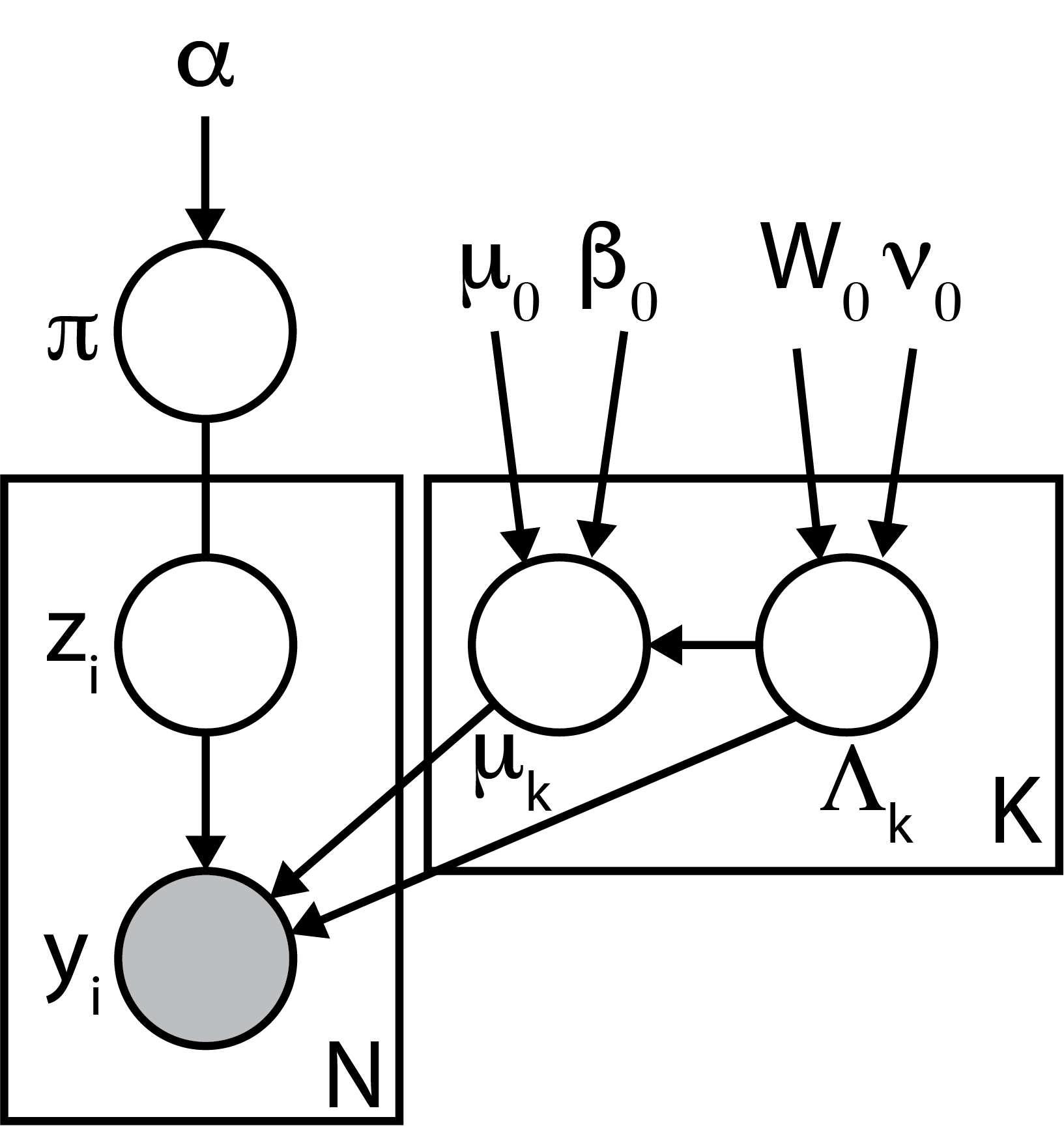}
	\caption{Graphical model representation of the BGMM model.}
	\label{fig:bgmm-graphical-model}
\end{figure}

The object of inference is the full set of latent variables $x = (z_1, \ldots, z_N, \mu_1, \ldots, \mu_K, \Lambda_1, \ldots, \Lambda_K, \pi_1, \ldots, \pi_K)$.
The value of $z_i$ gives the cluster assignment for observation $y_i$, the values of $\mu_k$ and $\Lambda_k$ give the cluster location and spread, and the value of $\pi_k$ gives the relative proportion of the data in cluster $k$.
The prior, $p(x|\phi)$ is parameterized by $\phi = (\alpha_0, \beta_0, W_0, \nu_0)$.
Given this model, the posterior density function factorizes as
\begin{equation}
	p(x | y, \phi) \propto p(x,y,\phi) = p(y | z, \mu, \Lambda) p(z | \pi) p(\pi ; \alpha_0) p(\mu | \Lambda; \beta_0) p(\Lambda; W_0, \nu_0).
\end{equation}
and can be represented as a factor graph (for full derivations, see Appendix~\ref{sec:bgmm_factor_form}).

A graphical representation of the factor graph is shown in Figure~\ref{fig:bgmm-factor-model}.

\begin{figure}[h!]
	\centering
	\includegraphics[width=0.9\linewidth]{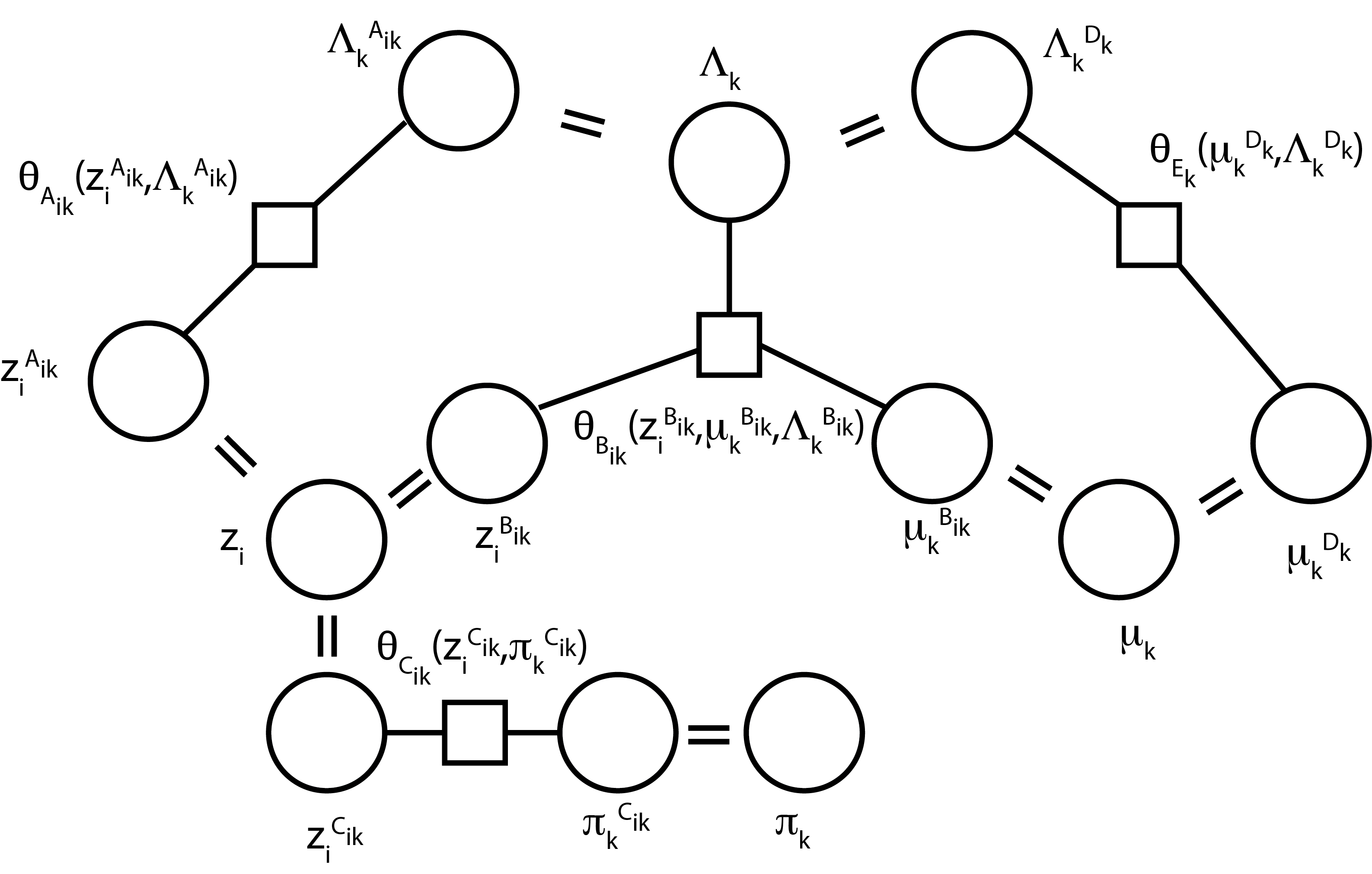}
	\caption{Factor graph representation of the BGMM model.}
	\label{fig:bgmm-factor-model}
\end{figure}

\begin{lemma}
The posterior of BGMM is not convex in the latent variables. (See Appendix~\ref{sec:bgmm-proof} for the proof.)
\end{lemma}
Thus, by Corollary~\ref{benders-corollary}, we note that our proposed method achieves local optimality. 
Even so, the value it converges to is quite good. 
\subsection{Experiments}
We compare our proposed approach to variational Bayes and the Gibbs sampler on three standard clustering data sets.

\paragraph{Data collection and preprocessing}
We obtained the Iris (\texttt{iris}, $n=150, d=4$), Wine Quality (\texttt{wine}, $n=178, d=13$), and Wisconsin Breast Cancer (\texttt{brca}, $n=569, d=30$) data sets from the UCI Machine Learning Repository \citep{dua2019uci}.
A 3-d projection of \texttt{iris} was obtained by projecting on the first three principal components, a 6-d projection of \texttt{wine} was obtained by projecting on the first six principal components, and only the following features were employed for the \texttt{brca} data set: worst area, worst smoothness, and mean texture, then standardized.
The reduction in dimensions was conducted to obtain not highly correlated ($< 0.7$) dimensions but to keep a good portion of overall variability ($> 0.85$).
Since our goal is to obtain the global \gls{map} clustering given the data set rather than a prediction, all of the data was used for clustering.

\paragraph{Experimental protocol}
We used the \texttt{scikit-learn} \citep{scikit-learn} implementation of variational Bayes and implemented the Gibbs sampler for the Bayesian Gaussian mixture model in \texttt{python} as per the conditionals derived in \citep{bruno2021auto}.
The convergence of the chains was assessed using standard trace plots~\citep{roy2020convergence} (See Appendix~\ref{sec:gibbs-appendix}).
The variational Bayes experiments were run using algorithms defined in \texttt{python/scikitlearn} and the Gibbs sampler experiments were run using generative models defined in section~\ref{sec:bgmm-problem} on 1,000 iterations (with a burn-in period of 100 iterations) in \texttt{python}.

Our approach was implemented in GAMS \citep{bisschop1982gams, gams2017manual}, a standard optimization problem modeling language by initializing it after running the branch and bound algorithm (BNB) for six hundred seconds and adding the constraints detailed in Appendix~\ref{sec:bgmm-proof}.  

\paragraph{Clustering Comparison Metrics}
Evaluating the quality of a clustering algorithm is inherently more challenging because more often than not there is no true label of the data; so we compare clustering solutions to one another rather than to a ground truth.
\citep{meilua2003comparing} proposed an information theoretic criterion for comparing two clustering solutions. 
Variation of Information (VoI) measures the amount of information lost and gained in changing from clustering $C$ to clustering $C'$: 
\begin{equation}
    \label{eqn:vid}
    \textrm{VoI}(C,C') = H(C) + H(C') - 2I(C,C'),
\end{equation}
where $$H(C)= -\sum_{k=1}^K P_C(k) \log P_C(k)$$ is the entropy associated with clustering $C$, and $$I(C, C') = \sum_{k=1}^K \sum_{k'=1}^{K} P_{C,C'}(k,k') \log \frac{P_{C,C'}(k,k')}{P_C(k)P_{C'}(k')}$$ is the mutual information between the clustering solutions.
In the entropy and mutual information computations, $P_{C,C'}(k,k') = \frac{|C_k \cap C'_{k'}|}{n}$ is the joint probability that a point belongs to $C_{k}$ and $C'_{k'}$, and $P_C(k) = \frac{|C_k|}{n}$ is proportion of points in cluster $k$.
Along with Variation of Information, we also report three additional quantities: $\log$ MAP (calculated using \ref{eq:term1} - \ref{eq:term6}), running time - the running time taken by each algorithm in seconds, and optimality guaranteed - whether optimality is achieved by each algorithm, which helps us understand the nature of these algorithms compared to our proposed method better. Running time and optimality guarantee analyses are provided in Appendix ~\ref{sec:runopt-appendix}.

\paragraph{Comparison to other algorithms}
Table~\ref{table:vi-tables} and Table~\ref{table:loglik-table} show comparisons of our proposed method (GBD) with other standard Bayesian estimation methods (variational Bayes, Gibbs sampler). 
Table~\ref{table:vi-tables} shows a comparison of inference algorithms using the variation of information (VoI) distance for three standard clustering test data sets: \texttt{iris}, \texttt{wine}, and \texttt{brca}.
A large VoI score indicates that the sample assignment configuration between two inference algorithms has little overlap and a small VoI score indicates a large degree of overlap. 
The VoI distance between GBD (our method) and the provided label is comparable to the distance between other inference methods and the label. 
In all tests, the GBD solution is closer to the label and closer to variational Bayes than the Gibbs sampler.
We note the caveat that labels are typically generated by human classification and the labels may be erroneous when the differences between groups are difficult to ascertain based on features alone.  

\begin{table}[h!]
    \begin{subtable}[b]{\textwidth}
        \centering
        \begin{tabular}{@{}lcccc@{}}
            \toprule
            & Label & Var Bayes & Sampler & GBD\\
            \midrule
            Label & 0 & 0.124 & 0.146 & 0.149  \\
            Var Bayes &  & 0 & 0.122 & 0.116  \\
            Sampler &  &  & 0 & 0.122 \\
            GBD &  & & &  0 \\
            \bottomrule
        \end{tabular}
        \caption{VoI distance for \texttt{iris} data set.}
    \end{subtable}\\[10pt]
    \begin{subtable}[b]{\textwidth}
        \centering
        \begin{tabular}{@{}lcccc@{}}
            \toprule
            & Label & Var Bayes & Sampler & GBD\\
            \midrule
            Label & 0 & 0.033 & 0.261 & 0.048  \\
            Var Bayes & & 0 & 0.252 & 0.032  \\
            Sampler & & & 0 & 0.264 \\
            GBD &  & & & 0 \\
            \bottomrule
        \end{tabular}
        \caption{VoI distance for \texttt{wine} data set.}
    \end{subtable}\\[10pt]
    \begin{subtable}[b]{\textwidth}
        \centering
        \begin{tabular}{@{}lcccc@{}}
            \toprule
            & Label & Var Bayes & Sampler & GBD\\
            \midrule
            Label & 0 & 0.097 & 0.206 & 0.080  \\
            Var Bayes & & 0 & 0.210 & 0.079  \\
            Sampler & & & 0 & 0.196 \\
            GBD &  & & & 0 \\
            \bottomrule
        \end{tabular}
        \caption{VoI distance for \texttt{brca} data set.}
    \end{subtable}
    \caption{Comparison of inference algorithms for the Bayesian Gaussian mixture model using variation of information (VoI) scores.}
    \label{table:vi-tables}
\end{table}

Table~\ref{table:loglik-table} shows the estimated $\log$ MAP values for the three different methods on all three data sets along with the log MAP values for branch and bound (BNB) since we use it to initialize our method.
In terms of log MAP, GBD identifies a MAP estimate that is better than the other inference methods.

\begin{table}[h!]
        \centering
        \begin{tabular}{@{}lccc@{}}
            \toprule
            & \texttt{iris} & \texttt{wine} & \texttt{brca} \\
            \midrule
            Var Bayes & -150.997 & -676.813 & -681.272 \\
            Sampler & -171.044 & -1172.079 & -1195.411 \\
            BNB (UB) & -90.149 (230925.3) & -646.203 (1156346.8) & -660.976 (127090.3) \\ 
            GBD (UB) & \textbf{-89.411 (-81.5)} & \textbf{-646.200 (-200.4)} & \textbf{-660.971 (-657.3)} \\
            \bottomrule
        \end{tabular}
        \caption{Comparison of inference algorithms for the Bayesian Gaussian mixture model using log MAP values. GBD outperforms both variational Bayes and the Gibbs sampler. }
        \label{table:loglik-table}
\end{table}

\paragraph{Impact of must-link constraints.}
One of the benefits of our methodology, aside from a certificate of convergence, is the capability to naturally add constraints. 
In Table~\ref{table:const-table}, we outline the log MAP values after adding 2, 4, 8, 16, and 32 must-link constraints on the \texttt{iris} data set. 
We also report the same values from the branch and bound algorithm after adding the same constraints.
The lower bound for GBD does not change except for the case with 16 constraints indicating that the optimal MAP is identified in all cases
The upper bound shows a downward trend indicating that the constraints improve the ability to provide a certificate of (local) optimality.
Even so, GBD outperforms BNB in all cases and is able to decrease the upper bound of the log MAP with the addition of new constraints.

\begin{table}[h!]
        \centering
        \begin{tabular}{@{}cccc@{}}
            \toprule
            Must-Link & GBD (UB) & BNB (UB) \\
            \midrule
             2 & -90.146 (-81.751) & -90.149 (230000.6) \\
             4 & -90.146 (-78.566) & -90.149 (229878.16) \\
             8 & -90.146 (-78.566) & -90.149 (228843.24) \\ 
             16 & -89.055 (-81.486) & -90.149 (227252.85) \\  
             32 & -90.146 (-78.089) & -90.149 (207962.17) \\
            \bottomrule
        \end{tabular}
        \caption{Comparison of log MAP values by adding a different number of must-link constraints for the \texttt{iris} data set.}
        \label{table:const-table}
\end{table}

\section{Latent Dirichlet Allocation}
\label{sec:lda_sec}
In this section, we formulate and derive a Generalized Benders decomposition algorithm for the Latent Dirichlet allocation (LDA) model~\citep{blei2003latent}.

\subsection{Problem Formulation}

A latent Dirichlet allocation model has the form:
\begin{eqnarray}
    \beta_k | \eta_0 &\sim& \text{Dirichlet}_V(\eta_0) \\
    \Theta_d | \alpha_0 &\sim& \text{Dirichlet}_K(\alpha_0) \\
    z_{dn} | \Theta_d &\sim& \text{Categorical}_K(\Theta_d) \\
    w_{dn} | z_{dn}, \beta_k &\sim& \text{Categorical}_V(\beta_{z_{dn}}).
\end{eqnarray}

\begin{figure}[h!]
	\centering
	\includegraphics[width=0.9\linewidth]{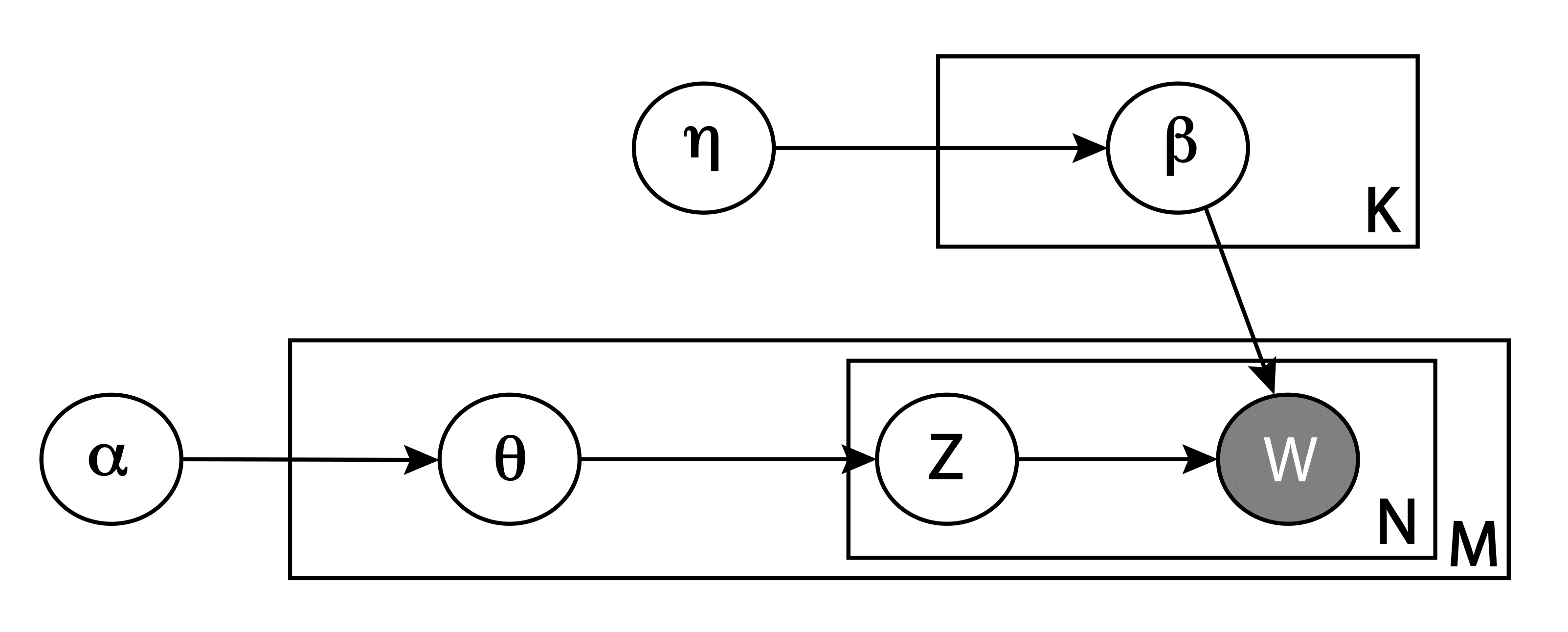}
	\caption{Graphical model representation of the smoothed LDA model.}
	\label{fig:lda-graphical-model}
\end{figure}

The object of inference is the full set of latent variables $x = (z_{11}, \ldots, z_{MN}, \Theta_1, \ldots, \Theta_M, \beta_1, \ldots, \beta_K)$.
The values of $\beta_k$ are the distribution of words in topic $k$, the values of $\Theta_d$ are the probability distribution of topics in document $d$, and the value of $z_{dn}$ gives the topic assignment of an observed $n$-th word in document $d$, $w_{dn}$.
The prior, $p(x|\phi)$ is parameterized by $\phi = (\alpha_0, \eta_0)$.
Given this model, the posterior density function factorizes as
\begin{equation}
	p(x | w, \phi) \propto p(x,w,\phi) = p(w | z, \beta) p(z | \Theta) p(\Theta ;\alpha_0) p(\beta;\eta_0)
\end{equation}	
and can be represented as a factor graph (for full derivations, see Appendix~\ref{sec:lda_factor_form}).

A graphical representation of the factor graph is shown in Figure~\ref{fig:lda-factor-model}.

\begin{figure}[h!]
	\centering
	\includegraphics[width=0.9\linewidth]{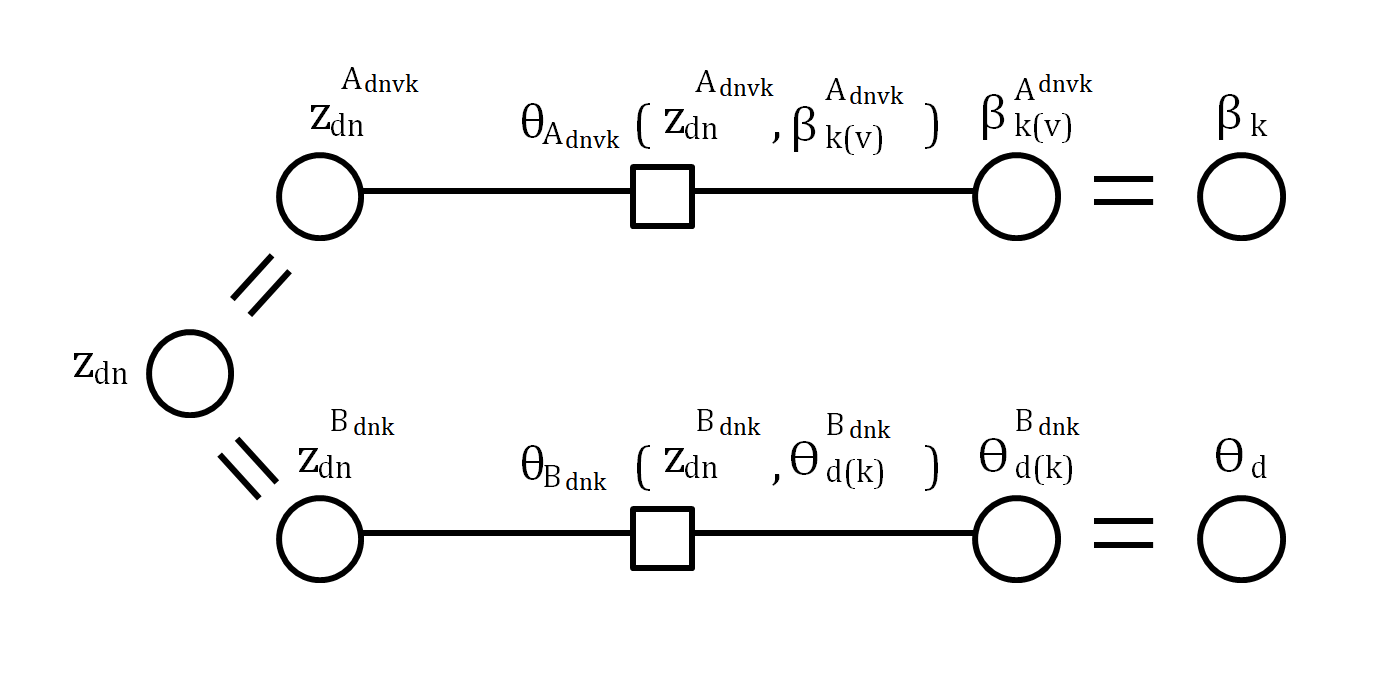}
	\caption{Factor graph representation of the smoothed LDA model.}
	\label{fig:lda-factor-model}
\end{figure}

\begin{lemma}
The posterior of BGMM is not convex in the latent variables. (See Appendix~\ref{sec:lda-proof} for the proof.)
\end{lemma}
Thus, by Corollary~\ref{benders-corollary}, we note that our proposed method may achieve a local optimality.

\subsection{Experiments}
We compare our proposed approach to variational Bayes and the Gibbs sampler on the 20 newsgroups data set.

\paragraph{Data collection and preprocessing}
We obtained the 20 newsgroups (\texttt{news20}, $n=18,846$, $d=1$) data set from the UCI KDD Archive \citep{ucikdd}. 
We only look at 50 randomly selected documents from the original data set, and run all methods under consideration on a list of 5 topics and 25 most frequently used words within the selected data.

\paragraph{Experimental protocol}
We used the \texttt{scikit-learn} \citep{scikit-learn} implementation of variational Bayes and implemented the Gibbs sampler for latent Dirichlet allocation model based on Griffiths and Steyvers ~\citep{griffiths2002}.
The variational Bayes experiment was run using algorithms defined in \texttt{python/scikitlearn}, and the Gibbs sampler experiments were run via \texttt{python} for 1,000 iterations (with a burn-in period of 100 iterations).
Our approach was implemented in GAMS \citep{bisschop1982gams, gams2017manual}, a standard optimization problem modeling language, with a solver BARON \citep{tawarmalani2005polyhedral, sahinidis2017baron}.
For each problem solved using BARON, we limit the process to end in 600 seconds.
Like the experiments for the BGMM, our approach was implemented by initializing it after running the branch and bound algorithm and adding the constraints detailed in Appendix~\ref{sec:lda-proof}.  

\paragraph{Comparison Metrics}
To compare our proposed method to variational Bayes and the Gibbs sampler we calculate the \textit{log MAP} - calculated using \ref{eq:term1_lda} - \ref{eq:term4_lda} and we also report the top ten words for two topics generated by the three different methods under consideration. 
To make sure the two topics that we select align across the three different methods, we sort the topics in descending order and then pick the top ten words from each topic according to their probabilities. 
Another important metric to compare LDA models is held-out test perplexity which has been analyzed for all three methods in Appendix ~\ref{sec:perp-appendix}.  

\paragraph{Comparison to other algorithms}

In terms of log MAP, GBD outperforms the other methods under consideration with a log MAP value of -1.10113E+6 as compared to -1.11275E+6 for variational Bayes and -1.1124280E+6 for the Gibbs sampler. 
Table \ref{table:topic-table} shows the top ten words from topic 1 and topic 2 generated by each of the three methods under consideration.
We observe that all three methods under consideration show similar performance. 
We highlight those words in Table \ref{table:topic-table} that are common between at least two of the three methods under consideration. 
For both topic 1 and topic 2 there is a high degree of overlap between the top ten selected words. 
One important thing to note however is that our proposed method provides a certificate of local optimality in finite time.
On the other hand, variational Bayes provides a locally optimal solution but to an approximation of the posterior, not to the exact posterior.
And, the Gibbs sampler does not provide an optimality guarantee for a finite sample size.
Also important to consider is the fact that our proposed method allows the addition of hard constraints based on some prior knowledge of the problem which other competing methods do not allow.

\begin{table}[]
    \centering
    \begin{tabular}{lcccccl}\toprule
    & \multicolumn{3}{c}{Topic 1} & \multicolumn{3}{c}{Topic 2}
    \\\cmidrule(lr){2-4}\cmidrule(lr){5-7}
& GBD            & Sampler        & Var Bayes        & GBD     & Sampler & Var Bayes\\\midrule
& \textbf{israel} & \textbf{like}  & \textbf{israel} & \textbf{edu}     & \textbf{key}     & just   \\
& \textbf{time}   & \textbf{time}  & key             & \textbf{mail}    & line    & \textbf{like}   \\
& \textbf{just}   & \textbf{just}  & \textbf{like}   & \textbf{send}    & \textbf{like}    & \textbf{time}   \\
& \textbf{like}   & data           & \textbf{time}   & 3d      & \textbf{mail}    & \textbf{mail}   \\
& \textbf{image}  & com            & \textbf{just}   & com     & \textbf{format}  & \textbf{key}    \\
& \textbf{format} & \textbf{format}& line            & \textbf{objects} & \textbf{time}      & help   \\
& 3d              & gov            & \textbf{image}  & image   & \textbf{send}    & \textbf{edu}    \\
& ac              & help   & message                 & file    & \textbf{objects} & gov    \\
& available       & objects & mail                   & \textbf{format}  & message & \textbf{israel} \\
& based           & send    & help                  & stuff   & \textbf{israel}  & com \\\bottomrule
    \end{tabular}
    \caption{Comparison of the top 10 words from Topic 1 and Topic 2 for all three methods}
    \label{table:topic-table}
\end{table}

\section{Discussion}
\label{sec:discussion}
In this section, we summarize the main contributions of our work and discuss some limitations.
In this paper, we first provide a computationally efficient algorithm for maximum a-posteriori inference in Bayesian models that yields an optimal value.
The method sequentially adds optimal constraints to the fully relaxed dual problem using Benders' decomposition.
Second, we show when the algorithm yields an $\epsilon$-globally optimal value and when it yields a locally optimal value.
Third, we derive and implement our method for two Bayesian factor models (BGMM and LDA) with the incorporation of hard domain constraints to reflect the properties of the models.
Finally, for each model, we evaluate the algorithm on standard data sets and compare the results to other standard methods.
In terms of log MAP, our proposed method outperforms other competing methods for both Bayesian factor models on real-world data sets.
For BGMM, the Variation of Information results show that the final clustering provided by our method is similar to the suggested labels.
And, for the LDA model, the comparison of the top ten words assigned for both topics shows a high degree of similarity between our proposed method and other competing methods.

\par
There are a few important points to note when applying our method though.
Firstly, because we turn our problem into an optimization problem, it allows us to
add hard constraints appropriately. 
For example, we can force the final solution to have two data points being clustered into the same cluster, or into different clusters based on prior knowledge. 
This flexibility allows our method to become quite useful as it is not always easy to enforce such constraints using other standard methods. 
However, adding such constraints may limit the domain space in a way that the optimization solver runs less efficiently.
Secondly, GBD is not guaranteed to always provide the global optimum and when it provides a local optimum, the solution may be influenced by the initial points chosen. 
Thus, it is important to provide good initial points to start with for the proposed method to work best.
Finally, as is the case with all Bayesian mixture models, it is important to set appropriate values for parameters of the prior probability distributions for parameter estimation and inference on posterior probability distribution. 
In many cases, one can determine them with prior knowledge or based on common practices in literature. 
However, when neither is available, one would need to come up with some methodology to choose the prior parameter values that are appropriate for the model and for the data set.

\section*{Acknowledgments}
This work was funded by the National Science Foundation TRIPODS award 1934846.

\bibliographystyle{plain} 
\bibliography{references}  

\begin{thebibliography}{10}

\bibitem{andrieu2003introduction}
Christophe Andrieu, Nando De~Freitas, Arnaud Doucet, and Michael~I Jordan.
\newblock An introduction to {MCMC} for machine learning.
\newblock {\em Machine learning}, 50:5--43, 2003.

\bibitem{angel2022cluster}
Rico Angell, Nicholas Monath, Nishant Yadav, and Andrew Mc{C}allum.
\newblock Interactive correlation clustering with existential cluster constraints.
\newblock In {\em Proceedings of the 39th International Conference on Machine Learning}, volume 162 of {\em Proceedings of Machine Learning Research}, pages 703--716. PMLR, 17--23 Jul 2022.

\bibitem{asuncion2009smoothing}
Arthur Asuncion, Max Welling, Padhraic Smyth, and Yee~Whye Teh.
\newblock On smoothing and inference for topic models.
\newblock In {\em Proceedings of the Twenty-Fifth Conference on Uncertainty in Artificial Intelligence}, pages 27--34, 2009.

\bibitem{balkrishnan2017EM}
Sivaraman Balakrishnan, Martin~J. Wainwright, and Bin Yu.
\newblock Statistical guarantees for the {EM} algorithm: From population to sample-based analysis.
\newblock {\em The Annals of Statistics}, 45(1):77 -- 120, 2017.

\bibitem{baumann2020binary}
Philipp Baumann.
\newblock A binary linear programming-based {K}-means algorithm for clustering with must-link and cannot-link constraints.
\newblock In {\em 2020 IEEE International Conference on Industrial Engineering and Engineering Management (IEEM)}, pages 324--328, 2020.

\bibitem{bertsimas2017logistic}
Dimitris Bertsimas and Angela King.
\newblock Logistic regression: From art to science.
\newblock {\em Statistical Science}, 32(3):367--384, 2017.

\bibitem{bertsimas2005optimization}
Dimitris Bertsimas and Robert Weismantel.
\newblock {\em Optimization over integers}.
\newblock {Dynamic Ideas}, {Belmont}, 2005.

\bibitem{bishop2006pattern}
Christopher~M Bishop and Nasser~M Nasrabadi.
\newblock {\em Pattern recognition and machine learning}, volume~4.
\newblock Springer, 2006.

\bibitem{bisschop1982gams}
Johannes Bisschop and Alexander Meeraus.
\newblock On the development of a general algebraic modeling system in a strategic planning environment.
\newblock In {\em Applications}, volume~20 of {\em Mathematical Programming Studies}, pages 1--29. Springer Berlin Heidelberg, 1982.

\bibitem{blei2017variational}
David~M Blei, Alp Kucukelbir, and Jon~D McAuliffe.
\newblock Variational inference: A review for statisticians.
\newblock {\em Journal of the American statistical Association}, 112(518):859--877, 2017.

\bibitem{blei2003latent}
David~M Blei, Andrew~Y Ng, and Michael~I Jordan.
\newblock Latent dirichlet allocation.
\newblock {\em Journal of machine Learning research}, 3(Jan):993--1022, 2003.

\bibitem{brooks1998general}
Stephen~P Brooks and Andrew Gelman.
\newblock General methods for monitoring convergence of iterative simulations.
\newblock {\em Journal of computational and graphical statistics}, 7(4):434--455, 1998.

\bibitem{bruno2021auto}
Mattia Bruno.
\newblock {\em Auto-encoding variational {B}ayes}.
\newblock PhD thesis, Università degli Studi di Padova, 2021.

\bibitem{chib1995understanding}
Siddhartha Chib and Edward Greenberg.
\newblock Understanding the {M}etropolis-{H}astings algorithm.
\newblock {\em The American Statistician}, 49(4):327--335, 1995.

\bibitem{clausen1999branch}
Jens Clausen.
\newblock Branch and bound algorithms-principles and examples.
\newblock Technical report, Department of Computer Science, University of Copenhagen, 1999.

\bibitem{davidson2005clustering}
Ian Davidson and SS~Ravi.
\newblock Clustering with constraints: Feasibility issues and the k-means algorithm.
\newblock In {\em Proceedings of the 2005 SIAM international conference on data mining}, pages 138--149. SIAM, 2005.

\bibitem{dua2019uci}
Dheeru Dua and Casey Graff.
\newblock {UCI} machine learning repository, 2017.

\bibitem{duran1986outer}
Marco~A Duran and Ignacio~E Grossmann.
\newblock An outer-approximation algorithm for a class of mixed-integer nonlinear programs.
\newblock {\em Mathematical programming}, 36:307--339, 1986.

\bibitem{duran1984mixed}
Marco Duran-Pena.
\newblock {\em A mixed-integer nonlinear programming approach for the systematic synthesis of engineering systems}.
\newblock PhD thesis, Department of Chemical Engineering, Carnegie-Mellon University, 1984.

\bibitem{floudas2013deterministic}
Christodoulos~A Floudas.
\newblock {\em Deterministic global optimization: theory, methods and applications}, volume~37.
\newblock Springer Science \& Business Media, 2013.

\bibitem{gams2017manual}
{GAMS Development Corporation}.
\newblock {\em GAMS --- A User's Guide}.
\newblock GAMS Development Corporation, 2017.

\bibitem{gelman1992inference}
Andrew Gelman and Donald~B Rubin.
\newblock Inference from iterative simulation using multiple sequences.
\newblock {\em Statistical science}, 7(4):457--472, 1992.

\bibitem{geman1984gibbs}
Stuart Geman and Donald Geman.
\newblock Stochastic relaxation, gibbs distributions, and the bayesian restoration of images.
\newblock {\em IEEE Transactions on Pattern Analysis and Machine Intelligence}, PAMI-6(6):721--741, 1984.

\bibitem{geoffrion1972generalized}
Arthur~M. Geoffrion.
\newblock Generalized benders decomposition.
\newblock {\em Journal of Optimization Theory and Applications}, 10, 1972.

\bibitem{gilmore1961linear}
Paul~C Gilmore and Ralph~E Gomory.
\newblock A linear programming approach to the cutting-stock problem.
\newblock {\em Operations research}, 9(6):849--859, 1961.

\bibitem{goemans1995improved}
Michel~X Goemans and David~P Williamson.
\newblock Improved approximation algorithms for maximum cut and satisfiability problems using semidefinite programming.
\newblock {\em Journal of the ACM (JACM)}, 42(6):1115--1145, 1995.

\bibitem{griffiths2002}
Thomas Griffiths and Mark Steyvers.
\newblock Prediction and semantic association.
\newblock {\em Advances in Neural Information Processing Systems (NeurIPS)}, 2002.

\bibitem{hajela1990genetic}
Prabhat Hajela.
\newblock Genetic search-an approach to the nonconvex optimization problem.
\newblock {\em AIAA journal}, 28(7):1205--1210, 1990.

\bibitem{heinrich2005parameter}
Gregor Heinrich.
\newblock Parameter estimation for text analysis.
\newblock Technical report, Fraunhofer IGD, 2005.

\bibitem{ucikdd}
S.~Hettich and S.~D. Bay.
\newblock {UCI} kdd archive, 1999.

\bibitem{kronqvist2019review}
Jan Kronqvist, David~E. Bernal, Andreas Lundell, and Ignacio~E. Grossmann.
\newblock A review and comparison of solvers for convex {MINLP}.
\newblock {\em Optimization and Engineering}, 20(2):397--455, Jun 2019.

\bibitem{land1960automatic}
A.~H. Land and A.~G. Doig.
\newblock An automatic method of solving discrete programming problems.
\newblock {\em Econometrica}, 28(3):497, 1960.

\bibitem{meilua2003comparing}
Marina Meil{\u{a}}.
\newblock Comparing clusterings by the variation of information.
\newblock In {\em Learning Theory and Kernel Machines: 16th Annual Conference on Learning Theory and 7th Kernel Workshop, COLT/Kernel 2003, Washington, DC, USA, August 24-27, 2003. Proceedings}, pages 173--187. Springer, 2003.

\bibitem{murphy2012machine}
Kevin~P Murphy.
\newblock {\em Machine learning: a probabilistic perspective}.
\newblock MIT press, 2012.

\bibitem{neal2011mcmc}
Radford~M Neal et~al.
\newblock {MCMC} using {H}amiltonian dynamics.
\newblock {\em Handbook of {M}arkov {C}hain {M}onte {C}arlo}, 2(11):2, 2011.

\bibitem{papanikolaou2017dense}
Yannis Papanikolaou, James~R Foulds, Timothy~N Rubin, and Grigorios Tsoumakas.
\newblock Dense distributions from sparse samples: improved gibbs sampling parameter estimators for lda.
\newblock {\em The Journal of Machine Learning Research}, 18(1):2058--2115, 2017.

\bibitem{pati2018statistical}
Debdeep Pati, Anirban Bhattacharya, and Yun Yang.
\newblock On statistical optimality of variational bayes.
\newblock In {\em International Conference on Artificial Intelligence and Statistics}, pages 1579--1588. PMLR, 2018.

\bibitem{scikit-learn}
F.~Pedregosa, G.~Varoquaux, A.~Gramfort, V.~Michel, B.~Thirion, O.~Grisel, M.~Blondel, P.~Prettenhofer, R.~Weiss, V.~Dubourg, J.~Vanderplas, A.~Passos, D.~Cournapeau, M.~Brucher, M.~Perrot, and E.~Duchesnay.
\newblock Scikit-learn: Machine learning in {P}ython.
\newblock {\em Journal of Machine Learning Research}, 12:2825--2830, 2011.

\bibitem{roy2020convergence}
Vivekananda Roy.
\newblock Convergence diagnostics for markov chain monte carlo.
\newblock {\em Annual Review of Statistics and Its Application}, 7:387--412, 2020.

\bibitem{sahinidis2017baron}
N.~V. Sahinidis.
\newblock {\em {BARON 17.8.9: Global Optimization of Mixed-Integer Nonlinear Programs, {\em User's Manual}}}, 2017.

\bibitem{shimony1994finding}
Solomon~Eyal Shimony.
\newblock Finding maps for belief networks is {NP}-hard.
\newblock {\em Artificial intelligence}, 68(2):399--410, 1994.

\bibitem{sminchisescu2003mode}
Cristian Sminchisescu, Max Welling, and Geoffrey Hinton.
\newblock A mode-hopping {MCMC} sampler.
\newblock Technical Report CSRG-478, University of Toronto, 2003.

\bibitem{sontag2010introduction}
David Sontag, Amir Globerson, and Tommi Jaakkola.
\newblock Introduction to dual decomposition for inference.
\newblock In {\em Optimization for Machine Learning}, pages 219--254. MIT Press, 2012.

\bibitem{Sontag2008message}
David Sontag, Talya Meltzer, Amir Globerson, Yair Weiss, and Tommi Jaakkola.
\newblock Tightening {LP} relaxations for {MAP} using message-passing.
\newblock In {\em 24th Conference in Uncertainty in Artificial Intelligence}, pages 503--510. AUAI Press, 2008.

\bibitem{szu1987nonconvex}
H.H. Szu and R.L. Hartley.
\newblock Nonconvex optimization by fast simulated annealing.
\newblock {\em Proceedings of the IEEE}, 75(11):1538--1540, 1987.

\bibitem{tawarmalani2005polyhedral}
Mohit Tawarmalani and Nikolaos~V. Sahinidis.
\newblock A polyhedral branch-and-cut approach to global optimization.
\newblock {\em Mathematical Programming}, 103(2):225--249, 2005.

\bibitem{tran2021bayes}
Minh-Ngoc Tran, Trong-Nghia Nguyen, and Viet-Hung Dao.
\newblock A practical tutorial on variational bayes.
\newblock {\em arXiv preprint arXiv:2103.01327}, 2021.

\bibitem{wainwright2008graphical}
Martin~J Wainwright, Michael~I Jordan, et~al.
\newblock Graphical models, exponential families, and variational inference.
\newblock {\em Foundations and Trends{\textregistered} in Machine Learning}, 1(1--2):1--305, 2008.

\bibitem{wallach2009evaluation}
Hanna~M Wallach, Iain Murray, Ruslan Salakhutdinov, and David Mimno.
\newblock Evaluation methods for topic models.
\newblock In {\em Proceedings of the 26th annual international conference on machine learning}, pages 1105--1112, 2009.

\bibitem{wang2019frequentist}
Yixin Wang and David~M Blei.
\newblock Frequentist consistency of variational bayes.
\newblock {\em Journal of the American Statistical Association}, 114(527):1147--1161, 2019.

\bibitem{zymnis2008mixed}
Argyrios Zymnis, Stephen Boyd, and Dimitry Gorinevsky.
\newblock Mixed state estimation for a linear gaussian markov model.
\newblock In {\em 2008 47th IEEE Conference on Decision and Control}, pages 3219--3226. IEEE, 2008.

\end{thebibliography}

\clearpage
\appendix

\section{BGMM posterior density in a factor graph form}
\label{sec:bgmm_factor_form}

\paragraph{Term 1.}
The density function of $p(y | z, \mu, \Lambda)$ is
\begin{equation}
	p(y | z, \mu, \Lambda) = \prod_{i=1}^N \prod_{k=1}^K \left[\mathcal{N}(y_i; \mu_k, \Lambda_k^{-1})\right]^{z_{ik}}.
\end{equation}
The log-density function is
\begin{eqnarray*}
	\log p(y | z, \mu, \Lambda) &=& \sum_{i=1}^N \sum_{k=1}^K z_{ik} \left[ -\frac{D}{2} \log (2\pi) -\frac{1}{2} \log \det \Lambda_k^{-1} -\frac{1}{2} (y_i - \mu_k)^T \Lambda (y_i - \mu_k) \right] \\
	&+=&  \sum_{i=1}^N \sum_{k=1}^K z_{ik} \left[  \frac{1}{2} \log \det \Lambda_k -\frac{1}{2} y_i^T \Lambda_k y_i  + \mu_k^T \Lambda_k y_i  - \frac{1}{2} \mu_k^T \Lambda_k \mu_k \right] \\
	&=&  \frac{1}{2}  \sum_{i=1}^N \sum_{k=1}^K z_{ik}  \log \det \Lambda_k 
	- \frac{1}{2} \sum_{i=1}^N \sum_{k=1}^K z_{ik}  y_i^T \Lambda_k y_i 
	+  \sum_{i=1}^N \sum_{k=1}^K z_{ik} \mu_k^T \Lambda_k y_i  
	- \frac{1}{2} \sum_{i=1}^N \sum_{k=1}^K z_{ik} \mu_k^T \Lambda_k \mu_k
\end{eqnarray*}
The first factor, $z_{ik} \log \det \Lambda_k$, only involves the latent variables $x_{A_{ik}} = \{z_{ik}, \Lambda_k\}$.
Grouping the terms from the density that only involve those terms gives
\begin{equation}
	\theta_{A_{ik}} (x_{A_{ik}}) := \frac{1}{2} z_{ik} \left[ \log \det \Lambda_k - y_i^T \Lambda_k y_i \right].
\end{equation}
Similarly, the term $z_{ik} \mu_k^T \Lambda_k y_i$ involves latent variables $x_{B_{ik}} = \{z_{ik}, \mu_k, \Lambda_k\}$.
Grouping all of the terms that involve $x_{B_{ik}}$ gives
\begin{equation}
	\theta_{B_{ik}}(x_{B_{ik}}) := z_{ik} \left[ \mu_k^T \Lambda_k y_i - \frac{1}{2} \mu_k^T \Lambda_k \mu_k \right].
\end{equation}
Now, the density function can be written in factor form as
\begin{equation}
	\log p(y | z, \mu, \Lambda) = \sum_{f \in \mathcal{F}_1} \theta_f(x_f),
\end{equation}
where $\mathcal{F}_1 = \{A_{ik}, B_{ik}\}_{i=1,k=1}^{N,K}$.

\paragraph{Term 2.}

The log-density function of $p(z | \pi)$ is
\begin{equation}
	\log p(z | \pi) = \sum_{i=1}^N \sum_{k=1}^K z_{ik} \log \pi_k.
\end{equation}
The factors only involve latent variables $x_{C_{ik}} = \{z_{ik}, \pi_k\}$ and the factors are
\begin{equation}
	\theta_{C_{ik}}(x_{C_{ik}}) := z_{ik} \log \pi_k.
\end{equation}
The log-density function can then be written as
\begin{equation}
	\log p(z | \pi) = \sum_{f \in \mathcal{F}_2} \theta_f(x_f),
\end{equation}
where $\mathcal{F}_2 = \{C_{ik}\}_{i=1,k=1}^{N,K}$.

\paragraph{Term 3.}
The log-density function of $p(\pi | \alpha_0)$ is
\begin{equation}
	\log p(\pi | \alpha_0) = \log \Gamma(\sum_{k=1}^k \alpha_{0k}) - \sum_{k=1}^K \log \Gamma(\alpha_{0k}) + \sum_{k=1}^K (\alpha_{0k} - 1) \log \pi_k.
\end{equation}
The factors in this log-density only involve singleton latent variables, so they can be written as
\begin{equation}
	\theta_{\pi_k} = (\alpha_{0k} - 1) \log \pi_k
\end{equation}
and the log-density is
\begin{equation}
	\log p(\pi | \alpha_0) = \sum_{v \in \mathcal{V}_1} \theta_v(v),
\end{equation}
where $\mathcal{V}_1 = \{\pi_1, \ldots, \pi_K\}$.

\paragraph{Term 4.}
The log-density function is
\begin{eqnarray*}
	\log p(\mu | \Lambda; \beta_0) &=& \sum_{k=1}^K  \left[ -\frac{D}{2} \log (2\pi) -\frac{1}{2} \log \det (\beta_0 \Lambda_k)^{-1} -\frac{\beta_0}{2} (\mu_k - \mu_0)^T \Lambda (\mu_k - \mu_0) \right] \\
	&+=&  \sum_{k=1}^K  \left[  \frac{1}{2} \log \det (\beta_0 \Lambda_k) -\frac{\beta_0}{2} \mu_k^T \Lambda_k \mu_k  + \beta_0  \mu_k^T \Lambda_k \mu_0  - \frac{\beta_0}{2} \mu_0^T \Lambda_k \mu_0 \right] \\
	&=&  \frac{1}{2} \sum_{k=1}^K \log \det ( \beta_0 \Lambda_k )
	- \frac{\beta_0}{2} \sum_{k=1}^K  \mu_k^T \Lambda_k \mu_k 
	+  \beta_0 \sum_{k=1}^K  \mu_k^T \Lambda_k \mu_0
	- \frac{\beta_0}{2} \sum_{k=1}^K \mu_0^T \Lambda_k \mu_0.
\end{eqnarray*}
The first term only involves a single latent variable and the log-density can be reformulated by defining
\begin{equation}
	\theta_{\Lambda_k} := \frac{1}{2}\log \det ( \beta_0 \Lambda_k ) - \frac{\beta_0}{2} \sum_{k=1}^K \mu_0^T \Lambda_k \mu_0
\end{equation}
The second term involves $x_{D_k} = \{\mu_k, \Lambda_k\}$. 
Grouping all the terms involving those latent variables gives
 \begin{equation}
	\theta_{D_{k}}(x_{D_{k}}) := - \frac{\beta_0}{2}  \mu_k^T \Lambda_k \mu_k +  \beta_0  \mu_k^T \Lambda_k \mu_0.
 \end{equation}
The log-density function can then be written as
\begin{equation}
	\log p(\mu | \Lambda; \beta_0) =  \sum_{v \in \mathcal{V}_2} \theta_v(x_v) + \sum_{f \in \mathcal{F}_3} \theta_f(x_f),
\end{equation}
where $\mathcal{V}_2 =  \{\Lambda_1, \ldots, \Lambda_K\}$ and $\mathcal{F}_3 = \{D_{k}\}_{k=1}^{K}$.

\paragraph{Term 5.}
The log-density function of $p(\Lambda)$ is
\begin{equation*}
	\log p(\Lambda; W_0, \nu_0) = \sum_{k=1}^K \log B(W_0, \nu_0) + \frac{(\nu_0 - D - 1)}{2} \log \det \Lambda_k - \frac{1}{2} \Tr (W_0^{-1} \Lambda_k),
\end{equation*}
where
\begin{equation*}
	B(W_0, \nu_0) = \left[ \det W_0 ^{\nu_0/2}  2^{\nu_0 D / 2} \pi^{D (D-1)/4} \prod_{d=1}^D \Gamma \left( \frac{\nu_0 + 1 - d}{2} \right) \right]^{-1}.
\end{equation*}
The factors in this log-density only involve singleton latent variables, so they can be written as
\begin{equation}
	\theta_{\Lambda_k} += \frac{\nu_0 - D - 1 }{2}  \log \det \Lambda_k - \frac{1}{2} \Tr \left( W_0^{-1} \Lambda_k \right)
\end{equation}
and the log-density is
\begin{equation}
	\log p(\Lambda ; W_0, \nu_0) = \sum_{v \in \mathcal{V}_2} \theta_v(v),
\end{equation}
where $\mathcal{V}_2 = \{\Lambda_1, \ldots, \Lambda_K\}$.

\paragraph{Combining factors.}
Recall $\phi = (\alpha_0, W_0, \nu_0, \mu_0, \beta_0)$,\\ $x = (z_{11}, \ldots, z_{NK}, \mu_1, \ldots, \mu_K, \Lambda_1, \ldots, \Lambda_K, \pi_1, \ldots, \pi_K)$, and $y = (y_1, \ldots, y_N)$.
Combining the factorized representations of each of the terms in the joint log-density gives
\begin{equation}
	p(x | y, \phi) = f(x,\theta) = \sum_{v \in \mathcal{V}} \theta_v(x_v) + \sum_{f \in \mathcal{F}} \theta_f(x_f),
\end{equation}
where $\mathcal{V} =  \{z_{11}, \ldots, z_{NK}, \mu_1, \ldots, \mu_K, \Lambda_1, \ldots, \Lambda_K, \pi_1, \ldots, \pi_K\}$, and \\ $\mathcal{F} = \{[A_{ik}]_{k=1, i=1}^{K,N}, [B_{ik}]_{k=1, i=1}^{K,N}, [C_{ik}]_{k=1, i=1}^{K,N}, [D_{k}]_{k=1}^{K}\}$.
The factors are
\begin{eqnarray}
	\theta_{\pi_k}(\pi_k) &=& (\alpha_{0k} - 1) \log \pi_k \label{eq:term1}, \\
	\theta_{\Lambda_k}(\Lambda_k) &=& \frac{1}{2} \log \det ( \beta_0 \Lambda_k ) - \frac{\beta_0}{2} \sum_{k=1}^K \mu_0^T \Lambda_k \mu_0 \nonumber \\
	&& \quad + \frac{\nu_0 - D - 1 }{2} \log \det \Lambda_k - \frac{1}{2} \Tr \left( W_0^{-1} \Lambda_k \right) \label{eq:term2}, \\
	\theta_{A_{ik}} (x_{A_{ik}}) &=& \frac{1}{2} z_{ik} \left[ \log \det \Lambda_k - y_i^T \Lambda_k y_i \right] \label{eq:term3}, \\
	\theta_{B_{ik}}(x_{B_{ik}}) &=& z_{ik} \left[ \mu_k^T \Lambda_k y_i - \frac{1}{2} \mu_k^T \Lambda_k \mu_k \right] \label{eq:term4}, \\
	\theta_{C_{ik}}(x_{C_{ik}}) &=& z_{ik} \log \pi_k \label{eq:term5}, \\
	\theta_{D_{k}}(x_{D_{k}}) &=& - \frac{\beta_0}{2}  \mu_k^T \Lambda_k \mu_k +  \beta_0  \mu_k^T \Lambda_k \mu_0 \label{eq:term6}.
\end{eqnarray}
The non-singleton factors are $\theta_{A_{ik}} (x_{A_{ik}}), \theta_{B_{ik}}(x_{B_{ik}}), \theta_{C_{ik}}(x_{C_{ik}})$ and $\theta_{D_{k}}(x_{D_{k}})$.
There are $|\mathcal{F}| = 3NK+K$ non-singleton factors.
The dimension of the posterior space is $|x| = NK+3K$.

\section{BGMM in a factor graph form does not satisfy the conditions for Benders' decomposition}
\label{sec:bgmm-proof}

Recall the MAP inference problem for BGMM in factor graphs is
\begin{eqnarray*}
	\text{MAP}(\theta) &=& \max_{\bx, \bx^\mathcal{F}} \sum_{v \in \mathcal{V}} \theta_v(x_v) + \sum_{f \in \mathcal{F}} \sum_{v \in f} \theta_f(\bx_v^f) \\
	\text{subject to} && x_v = x_v^f,\quad \forall v \in f, f \in \mathcal{F},
\end{eqnarray*}
and generalized Bender's decomposition solves problems of the form
\begin{eqnarray*}
	\max_{x,y} & f(x,y) \\
	\text{subject to } & G(x,y) \geq 0, \quad x \in \mathcal{X}, y \in \mathcal{Y},
\end{eqnarray*}%
where the following conditions on the objective and constraints hold:
\begin{enumerate}
	\item for a fixed $y$, $f(x,y)$ separates into independent optimization problems each involving a different subvector of $x$,
	\item for a fixed $y$, $f(x,y)$ is of a special structure that can be solved efficiently, and 
	\item fixing $y$ renders the optimization problem concave in $x$.
\end{enumerate}

Let $x^\calF$ and $\bx$ in the MAP inference problem be $x$ and $y$ in the generalized Benders' decomposition problem respectively.

Assume that $\bx$ is fixed. If $\bx$ is fixed, the Hessian matrix, $\bf{H}$, is \\
\begin{center}
$\begin{pNiceMatrix}[first-row,first-col]
    & z_{ik}^{A_{ik}} & \Lambda_k^{A_{ik}} 
    & z_{ik}^{B_{ik}} & \mu_k^{B_{ik}} & \Lambda_k^{B_{ik}} 
    & z_{ik}^{C_{ik}} & \pi_k^{C_{ik}} 
    & \mu_k^{D_{ik}} & \Lambda_k^{D_{ik}} \\
z_{ik}^{A_{ik}}    & 0   & (1) & 0   & 0   & 0   & 0   & 0   & 0   & 0 \\
\Lambda_k^{A_{ik}} & (1) & (2) & 0   & 0   & 0   & 0   & 0   & 0   & 0 \\
z_{ik}^{B_{ik}}    & 0   & 0   & 0   & (3) & (5) & 0   & 0   & 0   & 0 \\
\mu_k^{B_{ik}}     & 0   & 0   & (3) & (4) & (6) & 0   & 0   & 0   & 0 \\
\Lambda_k^{B_{ik}} & 0   & 0   & (5) & (6) & 0   & 0   & 0   & 0   & 0 \\
z_{ik}^{C_{ik}}    & 0   & 0   & 0   & 0   & 0   & 0   & (7) & 0   & 0 \\
\pi_k^{C_{ik}}     & 0   & 0   & 0   & 0   & 0   & (7) & (8) & 0   & 0 \\
\mu_k^{D_{ik}}    & 0   & 0   & 0   & 0   & 0   & 0   & 0   & (9) & (10)\\
\Lambda_k^{D_{ik}} & 0   & 0   & 0   & 0   & 0   & 0   & 0   & (10)& 0 \\
\end{pNiceMatrix}$, where
\end{center}

\begin{eqnarray*}
	(1) = \frac{1}{2}\sum_{i=1}^{N}\sum_{k=1}^{K}\log\Tr(\textrm{adj} (\Lambda_k^{A_{ik}})) - \Tr(y_i y_i^T) &;& 
	(2) = \frac{\partial}
	    {\partial \Lambda_k^{A_{ik}}} \frac{1}{2}\sum_{i=1}^{N}\sum_{k=1}^{K}z_{ik}^{A_{ik}}(\log\Tr(\textrm{adj}( \Lambda_k^{A_{ik}})) - \Tr(y_i y_i^T)) \\
	(3) = \sum_{i=1}^{N}\sum_{k=1}^{K} \Lambda_k^{B_{ik}}y_i-{\mu_k^{B_{ik}}}^T\Lambda_k^{B_{ik}} &;&
	(4) = -\sum_{i=1}^{N}\sum_{k=1}^{K} z_{ik}^{B_{ik}}\Lambda_k^{B_{ik}}\\
	(5) = \sum_{i=1}^{N}\sum_{k=1}^{K} {\mu_k^{B_{ik}}}^Ty_i-\frac{1}{2}\Tr({\mu_k^{B_{ik}}}{\mu_k^{B_{ik}}}^T) &;&
	(6) = \sum_{i=1}^{N}\sum_{k=1}^{K} z_{ik}^{B_{ik}}(y_i-\mu_k^{B_{ik}})\\
	(7) = \sum_{i=1}^{N}\sum_{k=1}^{K} \frac{1}{\pi_k^{C_{ik}} } &;&
	(8) = -\sum_{i=1}^{N}\sum_{k=1}^{K} \frac{z_{ik}^{C_{ik}}}{\pi_k^{C_{ik}}}^2\\
	(9) = -\beta_0\sum_{k=1}^{K} \Lambda_k^{D_{ik}} &;&
	(10) = -\beta_0\sum_{k=1}^{K} \mu_k^{D_{ik}} - \mu_0,
\end{eqnarray*}
where $\textrm{adj}(A)$ is the adjugate of the matrix $A$.

Considering $\bf{H}$ as a block matrix, the eigenvalues of $\bf{H}$ are collections of each block's eigenvalues. Consider the first block as
\begin{center}
$\bf{H}_1=\{\bf{H}_{ij}: i \in {1,2}, j \in {1,2}\}=$
$\begin{bmatrix}
 0   & (1) \\
 (1) & (2) \\
\end{bmatrix}$.
\end{center}
Then the characteristic equation is
\begin{center}
$| \bf{H}_1-\lambda \bf{I}| = \left|
\begin{bmatrix}
 0   & (1) \\
 (1) & (2) \\
\end{bmatrix} - 
\begin{bmatrix}
 \lambda   & 0 \\
 0 & \lambda \\
\end{bmatrix} \right| = 
\left| \begin{bmatrix}
 -\lambda   & (1) \\
 (1) & (2)-\lambda \\
\end{bmatrix} \right| =
\lambda^2 -(2)\lambda -(1)^2 = 0$,
\end{center}
and the two eigenvalues are $\frac{(2)\pm\sqrt{(2)^2+4(1)^2}}{2}$. Because one of the eigenvalues, $\frac{(2)+\sqrt{(2)^2+4(1)^2}}{2}$, is always greater than 0, $\bf{H}$ is not negative definite and thus this violates the condition 3 of the generalized Bender's decomposition.

\section{LDA posterior density in a factor graph form}
\label{sec:lda_factor_form}

\paragraph{Term 1.}
The log-density function of $p(w | z, \beta)$ is
$\log p(w |  z, \beta) = \sum_{d=1}^M \sum_{n=1}^N \sum_{v=1}^V w_{dnv}\log\beta_{z_{dn}}$,
where $w_{dnv}$ is a binary representation of $v$-categorical variable of $w_{dn}$.

The term involves latent variables $x_{A_{dnk}} = \{z_{dn}, \beta_k\}$ and the factors are
$\theta_{A_{dnk}}(x_{A_{dnk}}) := w_{dn}\log\beta_{z_{dn}}$.
and the log-density can be written as
$\log p(w | z, \beta) = \sum_{f \in \mathcal{F}_1} \theta_f(x_f)$,
where $\mathcal{F}_1 = \{A_{dnk}\}_{d=1,n=1,k=1}^{M,N,K}$.

\paragraph{Term 2.}
The log-density function of $p(z|\Theta)$ is
$\log p(z|\Theta) = \sum_{d=1}^M \sum_{n=1}^N \sum_{k=1}^K z_{dnk}\log\Theta_{dk}$
where $z_{dnk}$ is a binary representation of $k$-categorical variable of $z_{dn}$.
The term involves latent variables $x_{B_{dnk}} = \{\Theta_{dk}, {z_{dnk}}\}$ and the factors are
$\theta_{B_{dnk}}(x_{B_{dnk}}) := z_{dnk} \log \Theta_{dk}$.
The log-density function can then be written as
$\log p(z | \Theta) = \sum_{f \in \mathcal{F}_2} \theta_f(x_f),$
where $\mathcal{F}_2 = \{B_{dnk}\}_{d=1,n=1,k=1}^{M,N,K}$.

\paragraph{Term 3.}
The log-density function of $p(\Theta ; \alpha_0)$ is
$\log p(\Theta ; \alpha_0) \propto \sum_{d=1}^M\sum_{k=1}^K (\alpha_{0k} - 1) \log \Theta_{dk}$.
The factors in this log-density only involve singleton latent variables, so they can be written as
$\theta_{\Theta_{dk}} = (\alpha_{0k} - 1) \log \Theta_{dk}$
and the log-density is
$\log p(\Theta ; \alpha_0) = \sum_{v \in \mathcal{V}_1} \theta_v(v)$,
where $\mathcal{V}_1 = \{\Theta_{11}, \ldots, \Theta_{MK}\}$.

\paragraph{Term 4.}
The log-density function of $p(\beta ; \eta_0)$ is
$\log p(\beta ; \eta_0) \propto \sum_{k=1}^K\sum_{v=1}^V (\eta_{0v} - 1) \log \beta_{kv}$.
The factors in this log-density only involve singleton latent variables, so they can be written as
$\theta_{\beta_{kv}} = (\eta_{0v} - 1) \log \beta_{kv}$
and the log-density is
$\log p(\beta ; \eta_0) = \sum_{v \in \mathcal{V}_2} \theta_v(v)$,
where $\mathcal{V}_2 = \{\beta_{11}, \ldots, \beta_{KV}\}$.

\paragraph{Combining factors}
Recall $\phi = (\alpha_0, \eta_0)$, $x = (z_{11}, \ldots, z_{MN}, \Theta_1, \ldots, \Theta_M, \beta_1, \ldots, \beta_K)$, and $w = (w_{11}, \ldots, w_{MN})$.
Combining the factorized representations of each of the terms in the joint log-density gives
\begin{equation}
	p(x | w, \phi) = f(x,\theta) = \sum_{v \in \mathcal{V}} \theta_v(x_v) + \sum_{f \in \mathcal{F}} \theta_f(x_f),
\end{equation}
where $\mathcal{V} =  \{\Theta_{11}, \ldots, \Theta_{MK}, \beta_{11}, \ldots, \beta_{KV}\}$, and $\mathcal{F} = \{[A_{dnk}]_{d=1, n=1, k=1}^{M,N,K}\, [B_{dnk}]_{d=1,n=1,k=1}^{M,N,K}\}$.
The factors are
\begin{eqnarray}
    \theta_{\beta_{kv}} &=& (\eta_{0v} - 1) \log \beta_{kv} \label{eq:term1_lda}, \\
	\theta_{\Theta_{dk}}(\Theta_{dk})&=& (\alpha_{0k} - 1) \log \Theta_{dk} \label{eq:term2_lda}, \\
	\theta_{A_{dnk}}(x_{A_{dnk}}) &=& w_{dn}\log\beta_{z_{dn}} \label{eq:term3_lda}, \\
	\theta_{B_{dnk}}(x_{B_{dnk}}) &=& z_{dnk} \log \Theta_{dk} \label{eq:term4_lda}.
\end{eqnarray}

\section{LDA in a factor graph form does not satisfy the conditions for Benders' decomposition}
\label{sec:lda-proof}

Recall the MAP inference problem for BGMM in factor graphs is
\begin{eqnarray*}
	\text{MAP}(\theta) &=& \max_{\bx, \bx^\mathcal{F}} \sum_{v \in \mathcal{V}} \theta_v(x_v) + \sum_{f \in \mathcal{F}} \sum_{v \in f} \theta_f(\bx_v^f) \\
	\text{subject to} && x_v = x_v^f,\quad \forall v \in f, f \in \mathcal{F},
\end{eqnarray*}
and generalized Bender's decomposition solves problems of the form
\begin{eqnarray*}
	\max_{x,y} & f(x,y) \\
	\text{subject to} & G(x,y) \geq 0, \quad x \in \mathcal{X}, y \in \mathcal{Y},
\end{eqnarray*}
where the following conditions on the objective and constraints hold:
\begin{enumerate}
	\item for a fixed $y$, $f(x,y)$ separates into independent optimization problems each involving a different subvector of $x$,
	\item for a fixed $y$, $f(x,y)$ is of a special structure that can be solved efficiently, and 
	\item fixing $y$ renders the optimization problem concave in $x$.
\end{enumerate}

Let $x^\calF$ and $\bx$ in the MAP inference problem be $x$ and $y$ in the generalized Benders' decomposition problem respectively.

Assume that $\bx$ is fixed. If $\bx$ is fixed, the Hessian matrix, $\bf{H}$, is \\
\begin{center}
$\begin{pNiceMatrix}[first-row,first-col]
   & z_{dn}^{A_{dnk}} & \beta_{k}^{A_{dnk}} & \Theta_{dk}^{B_{dnk}} & z_{dnk}^{B_{dnk}}  \\
z_{dn}^{A_{dnk}}        & (1) & (2) & 0 & 0 \\
\beta_{k}^{A_{dnk}}     & (2) & (3) & 0 & 0 \\
\Theta_{dk}^{B_{dnk}}   & 0 & 0 & (4) & (5) \\
z_{dnk}^{B_{dnk}}       & 0 & 0 & (5) & 0
\end{pNiceMatrix}$, where
\end{center}

\begin{eqnarray*}
    (1) &=& \sum_{d=1}^{M}\sum_{n=1}^{N}\sum_{v=1}^{V} w_{dnv}[\frac{z_{dn}''}{\beta_{z_{dn}}}-(z_{dn}')^2\beta_{z_{dn}}^{-2}] 
    \quad \text{(For simple notation '$A_{dnk}$' is omitted)}\\\\
    (2) &=& -\sum_{d=1}^{M}\sum_{n=1}^{N}\sum_{v=1}^{V} w_{dnv}z_{dn}'\beta_{z_{dn}}^{-2}
    \quad \text{(For simple notation '$A_{dnk}$' is omitted)}\\
    (3) &=& -\sum_{d=1}^{M}\sum_{n=1}^{N}\sum_{v=1}^{V} w_{dnv}\beta_{z_{dn}}^{-2}
    \quad \text{(For simple notation '$A_{dnk}$' is omitted)}\\
	(4) &=&-\sum_{d=1}^{M}\sum_{n=1}^{N}\sum_{k=1}^{K} \frac{z_{dnk}^{B_{dnk}} }{{\Theta_{dk}^{B_{dnk}}}^2} \\
	(5) &=& \sum_{d=1}^{M}\sum_{n=1}^{N}\sum_{k=1}^{K} \frac{1}{\Theta_{dk}^{B_{dnk}}}.
\end{eqnarray*}

Considering $\bf{H}$ as a block matrix, the eigenvalues of $\bf{H}$ are collections of each block's eigenvalues. Consider the second block as
\begin{center}
$\bf{H}_2= \{\bf{H}_{ij}: i \in {3,4}, j \in {3,4}\}=$
$\begin{bmatrix}
 (4) & (5) \\
 (5) & 0 \\
\end{bmatrix}$.
\end{center}
Then the characteristic equation is
\begin{center}
$| \bf{H}_2-\lambda \bf{I}| = \left|
\begin{bmatrix}
 (4) & (5) \\
 (5) & 0 \\
\end{bmatrix} - 
\begin{bmatrix}
 \lambda   & 0 \\
 0 & \lambda \\
\end{bmatrix} \right| = 
\left| \begin{bmatrix}
 (4)-\lambda   & (5) \\
 (5) & -\lambda \\
\end{bmatrix} \right| =
\lambda^2 -(4)\lambda -(5)^2 = 0$,
\end{center}
and the two eigenvalues are $\frac{(4)\pm\sqrt{(4)^2+4(5)^2}}{2}$. Because one of the eigenvalues, $\frac{(4)+\sqrt{(4)^2+4(5)^2}}{2}$, is always greater than 0, $\bf{H}$ is not negative definite and thus this violates the condition 3 of the generalized Bender's decomposition.

\section{Running time and Optimality for the BGMM}
\label{sec:runopt-appendix}
Table ~\ref{table:runopt-table} shows the running time and optimality guarantees for all three methods under consideration. 
The running time of the three methods shows a stark difference between them across all three data sets. 
variational Bayes is the fastest as expected and the Gibbs sampler, which is run for 1,000 iterations, is also quicker than GBD mainly because the number of data points in consideration is not very large. 
Our proposed method, while not as quick as some of the other methods presented, still performs the best in terms of log MAP and most importantly guarantees optimality which the other two methods don't.

\begin{table}
    \begin{subtable}[b]{\textwidth}
        \centering
        \begin{tabular}{@{}lcc@{}}
            \toprule
            & Running time (s) & Optimality Guaranteed \\
            \midrule
            Var Bayes & 0.064 & (N/A) \\
            Sampler & 74.250 & (N/A) \\
            GBD & 4074.856 & epoch2, iter6 \\
            \bottomrule
        \end{tabular}
        \caption{\texttt{iris}}
    \end{subtable}\\[10pt]
    \begin{subtable}[b]{\textwidth}
        \centering
        \begin{tabular}{@{}lcc@{}}
            \toprule
            & Running time (s) & Optimality Guaranteed  \\
            \midrule
            Var Bayes  & 0.017 & (N/A) \\
            Sampler  & 94.99 & (N/A) \\
            GBD  & 3238.427 & epoch4, iter4 \\
            \bottomrule
        \end{tabular}
        \caption{\texttt{wine}}
    \end{subtable}\\[10pt]
    \begin{subtable}[b]{\textwidth}
        \centering
        \begin{tabular}{@{}lcc@{}}
            \toprule
            & Running time (s) & Optimality Guaranteed  \\
            \midrule
            Var Bayes  & 0.044 & (N/A) \\
            Sampler  & 191.504 & (N/A) \\
            GBD & 9837.053 & epoch3, iter3 \\
            \bottomrule
        \end{tabular}
        \caption{\texttt{brca}}
    \end{subtable}
    \caption{Comparison of clustering methods by running time and optimality guarantee.}
    \label{table:runopt-table}
\end{table}

\section{Perplexity analysis for LDA}
\label{sec:perp-appendix}
The typical way to compare multiple LDA models involves assessing the probability of a set of held-out test documents, given an already trained LDA model \citep{wallach2009evaluation}. 
In our context, given estimates of $\beta_{k}$ (the distribution of words in topic $k$) and $\Theta_d$ (distribution of topics in document $d$), we calculate the posterior predictive likelihood of all words in the test set \citep{papanikolaou2017dense}. Given a set of test documents $D_{test}$ and an estimated model $(\beta, \Theta)$, the \textit{log-likelihood} is defined as \citep{heinrich2005parameter}:
$$l_{D_\textrm{test}}(\beta, \Theta) := \sum_{d=1}^{D_\textrm{test}} \log p(\mathbf{w}_{d}|\beta, \Theta) = \sum_{d=1}^{D_\textrm{test}}\sum_{i=1}^{N_{d}} \log \sum_{k=1}^{K}\beta_{kv} \Theta_{dk} $$
with $w_{di} = v$. 
Hence, the perplexity is
$$\textrm{perplexity} := \exp \left( -\frac{l_{D_\textrm{test}}(\beta,\Theta)}{\sum_{d=1}^{D_\textrm{test}} N_{d}} \right), $$
where lower values of perplexity signify a better model. In this case, $\mathbf{w}_d$ indicates the vector of word assignments in a document $d$, $v$ is a word type, and $N_d$ is the number of word tokens in document $d$. Important to note here is the fact that $\Theta_{d}$ is unknown for the set of test documents, $D_\textrm{test}$. Hence, following the methodology developed by \citep{asuncion2009smoothing}, we run the methods under consideration on the first half of each test document and then compute the perplexity of the held-out data (the second half of each document) based on the trained model's posterior predictive distribution. \\
For this experiment, we randomly sample 300 documents from the 20 newsgroups data set and run all methods under consideration on a list of 5 topics and 150 most frequently used words within the selected data.
Next, we assign ninety percent of the documents to our training set and ten percent to the test set.
Table \ref{table:perplex-table} shows the log-likelihood and the perplexity values for those thirty selected test documents, i.e. $D_{test} = 30$. 
For the LDA model, we observe that all three methods under consideration show similar performance. The log-likelihood and perplexity scores for each method are fairly close to one another which is also reflected in Table \ref{table:topic-table} where the top ten words assigned to two topics shows a high degree of similarity between the three methods under consideration.

\begin{table}
        \centering
        \begin{tabular}{@{}lcc@{}}
            \toprule
            & Log Likelihood & Perplexity\\
            \midrule
            Var Bayes & -2237.730 & 177.670 \\
            Sampler & -2220.319 & 170.652 \\
            GBD & -2271.184 & 191.976 \\
            \bottomrule
        \end{tabular}
        \caption{Comparison of variational Bayes, Gibbs sampler, and GBD using log-likelihood and perplexity}
        \label{table:perplex-table}
\end{table}

\section{Convergence Analysis for the BGMM Gibbs Sampler}
\label{sec:gibbs-appendix}
In this section, we present a convergence analysis for the Gibbs sampler of the BGMM model. 
We ran our sampler for 1,000 iterations with a burn-in period of 100 iterations. 
In figure \ref{fig:gibbs_pi} we observe that the sampled mixture weights for all three standard data sets converge very quickly over 1,000 iterations. 
Hence, to find the mode of the distribution over these 1,000 iterations, we multiply the sampled weights by a factor to ensure that we have integer values and then take the mode over those integer values.
Similarly, we observe in figures \ref{fig:gibbs_mu_iris}, \ref{fig:gibbs_mu_wine}, and \ref{fig:gibbs_mu_brca} that the sampled means for each component across all three data sets converge very quickly over 1,000 iterations.
In this case as well, to find the mode of the distribution over these 1,000 iterations, we multiply the sampled mean components by a factor to ensure that we have integer values and then take the mode over those integer values. 

\begin{figure}
     \centering
     \begin{subfigure}[b]{0.28\textwidth}
         \centering
         \includegraphics[width=\textwidth]{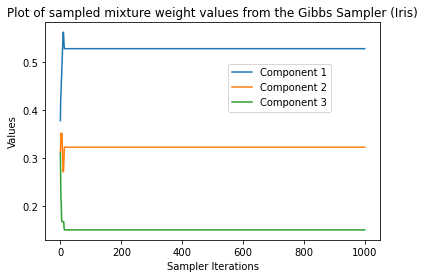}
         \caption{\texttt{iris}}
         \label{fig:iris_pi}
     \end{subfigure} %
     \begin{subfigure}[b]{0.28\textwidth}
         \centering
         \includegraphics[width=\textwidth]{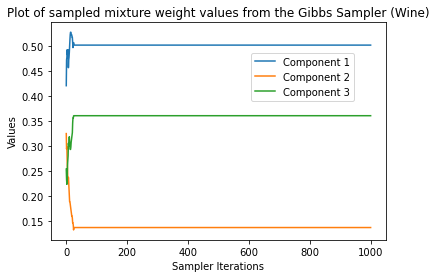}
         \caption{\texttt{wine}}
         \label{fig:wine_pi}
     \end{subfigure} %
     \begin{subfigure}[b]{0.28\textwidth}
         \centering
         \includegraphics[width=\textwidth]{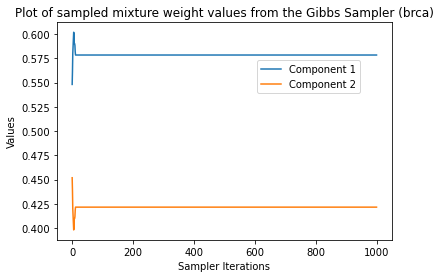}
         \caption{\texttt{brca}}
         \label{fig:brca_pi}
     \end{subfigure}
        \caption{Trace plots of sampled mixture weights from the Gibbs sampler for all three data sets}
        \label{fig:gibbs_pi}
\end{figure}

\begin{figure}[h!]
  \centering
    \includegraphics[width=0.72\textwidth]{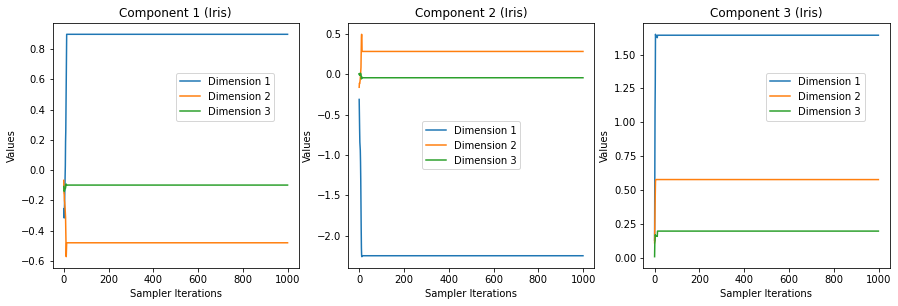}
  \caption{Trace plots of sampled mean components from the Gibbs sampler for \texttt{iris}}
  \label{fig:gibbs_mu_iris}
\end{figure}

\begin{figure}[h!]
  \centering
    \includegraphics[width=0.72\textwidth]{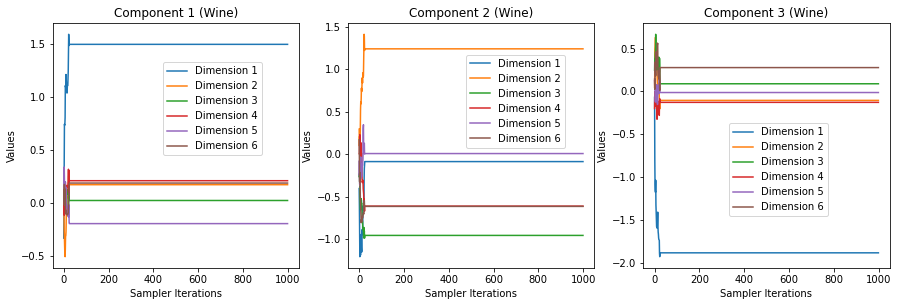}
  \caption{Trace plots of sampled mean components from the Gibbs sampler for \texttt{wine}}
  \label{fig:gibbs_mu_wine}
\end{figure}

\begin{figure}[h!]
  \centering
    \includegraphics[width=0.72\textwidth]{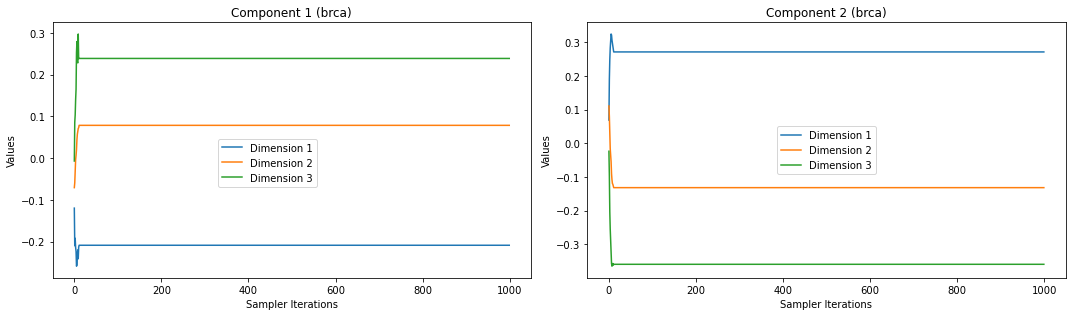}
  \caption{Trace plots of sampled mean components from the Gibbs Sampler for \texttt{brca}}
  \label{fig:gibbs_mu_brca}
\end{figure}

\end{document}